\documentclass{article}

\PassOptionsToPackage{numbers}{natbib}
\usepackage[final]{neurips_2020}

\usepackage[utf8]{inputenc} % allow utf-8 input
\usepackage[T1]{fontenc}    % use 8-bit T1 fonts
\usepackage[draft]{hyperref}       % hyperlinks
\usepackage{url}            % simple URL typesetting
\usepackage{booktabs}       % professional-quality tables
\usepackage{amsfonts}       % blackboard math symbols
\usepackage{nicefrac}       % compact symbols for 1/2, etc.
\usepackage{microtype}      % microtypography
\usepackage{amsmath} 		% added
\usepackage{amsthm}
\usepackage{algorithmic}
\usepackage{algorithm}
\makeatletter
\newcommand{\removelatexerror}{\let\@latex@error\@gobble}
\makeatother

\usepackage{graphicx} % added 
\usepackage{multirow} % added 

\graphicspath{{figs/}} % Directory in which figures are stored

\usepackage{subfigure} % added 

\usepackage{mathtools}
\DeclarePairedDelimiter{\norm}{\lVert}{\rVert}

\usepackage[most]{tcolorbox}

\newtcolorbox{incomplete}{skin=enhancedmiddle jigsaw,breakable,
  parbox=false,boxrule=0mm,leftrule=2mm,boxsep=0mm,
  arc=0mm,outer arc=0mm,left=3mm,right=3mm,top=1mm,bottom=1mm,
  toptitle=1mm,bottomtitle=1mm,oversize,colback=yellow!5!white,
  colframe=yellow}

\usepackage{xargs}

\usepackage[colorinlistoftodos,
            textsize=tiny,
            linecolor=red!30,
            bordercolor=red!30,
            backgroundcolor=red!10]{todonotes}

\newcommandx{\EB}[2][1=]{\todo[linecolor=red,backgroundcolor=red!25,bordercolor=red,#1]{\textbf{EB:} #2}}
\newcommandx{\LT}[2][1=]{\todo[linecolor=blue,backgroundcolor=blue!25,bordercolor=blue,#1]{\textbf{LT:} #2}}

\newcommand*{\field}[1]{\mathbb{\MakeUppercase{#1}}}		% scalar field
\newcommand*{\set}[1]{{\mathcal{\MakeUppercase{#1}}}}			% set symbol
\newcommand*{\R}{\field{R}} % field of real numbers, or the real line
\renewcommand{\vec}[1]{{\boldsymbol{\mathbf{#1}}}}
\newcommand*{\mat}[1]{\vec{\MakeUppercase{#1}}}
\newcommand*{\transpose}{\mathsf{T}}
\newcommand*{\dimension}{D}

\newcommand*{\expectation}{\mathbb{E}}

\newcommand*{\normal}{\mathcal{N}}					% normal distribution

\newcommand*{\vecMean}[1]{\vec{\hat{#1}}}   % mean of a random vector
\newcommand*{\observation}{y}
\newcommand*{\observations}{\vec{y}}

\newcommand*{\weights}{\vec{w}}         % feature weights
\newcommand*{\ffeature}{\vec{\phi}}     % Fourier feature vector
\newcommand*{\warp}{\vec{m}}
\newcommand*{\tfVec}{\vec{g}}
\newcommand*{\tfMat}{\mat{G}}
\newcommand*{\tfAdd}{\vec{h}}
\newcommand*{\InputSpace}{\set{X}}
\newcommand*{\WarpedSpace}{\set{Q}}
\newcommand*{\LinearOpSpace}{\set{L}}
\newcommand*{\upperGP}{u}
\newcommand*{\rawInput}{\vec{x}}
\newcommand*{\randomInput}{\vec{\tilde{x}}}

\newcommand*{\nObs}{N}

\newcommand*{\nFeatures}{M}

\title{Sparse Spectrum Warped Input Measures for Nonstationary Kernel Learning}

\author{%
  Anthony Tompkins$^1$\\
  \And
  Rafael Oliveira$^{1,2}$ \\
  \And
  Fabio Ramos$^{1,3}$ \\
  \AND
  $^1$ \textnormal{School of Computer Science, the University of Sydney, Australia}\\
  $^2$ \textnormal{ARC Centre for Data Analytics for Resources and Environments, Australia}\\
  $^3$ \textnormal{NVIDIA, USA}\\
}

\begin{document}

\maketitle

\begin{abstract}
  We establish a general form of explicit, input-dependent, measure-valued warpings for learning nonstationary kernels. While stationary kernels are ubiquitous and simple to use, they struggle to adapt to functions that vary in smoothness with respect to the input. The proposed learning algorithm warps inputs as conditional Gaussian measures that control the smoothness of a standard stationary kernel. This construction allows us to capture non-stationary patterns in the data and provides intuitive inductive bias. The resulting method is based on sparse spectrum Gaussian processes, enabling closed-form solutions, and is extensible to a stacked construction to capture more complex patterns. The method is extensively validated alongside related algorithms on synthetic and real world datasets. We demonstrate a remarkable efficiency in the number of parameters of the warping functions in learning problems with both small and large data regimes.
\end{abstract}

\section{Introduction}
Many interesting real world phenomena exhibit varying characteristics, such as smoothness, across their domain. Simpler phenomena that \textit{do not} exhibit such variation may be called \textit{stationary}. The typical kernel based learner canonically relies on a \textit{stationary} kernel function, a measure of "similarity", to define the prior beliefs over the function space. Such a kernel, however, cannot represent desirable \textit{nonstationary} nuances, like varying spatial smoothness and sudden discontinuities. Restrictive stationary assumptions do not generally hold and limit applicability to interesting problems, such as robotic control and reinforcement learning \cite{ng2006autonomous}, spatial mapping \cite{ton2018spatial}, genetics \cite{friedman2000using}, and Bayesian optimisation \cite{martinez2017bayesian}. One obvious way to alleviate the problem of finding the appropriate kernel function given one's data is hyperparameter optimisation. However for a GP with stationary kernel, even if the \textit{optimal} set of hyperparameters were found, it would be insufficient if our underlying response were nonstationary with respect to the observed inputs. %This provides a general motivation for the algorithmic contribution of this paper.

%As one of the ubiquitous methods for nonparametric Bayesian modelling, the original Gaussian process (GP) framework \cite{rasmussen2004gaussian} is a natural kernel machine. However, it typically has limits on scalability due to $\mathcal{O}(N^3)$ runtime complexity in the number of data points. 
In this paper we propose a method for nonstationary kernel learning, based on sparse spectral kernel representations. Our method is linear in complexity with respect to the number of data points and is simultaneously able to extract nonstationary patterns.
% The core aim of this paper is to introduce a novel approach to represent nonstationarity in sparse spectrum Gaussian process models through input warping. As the warping by itself can be considered as a latent function, we derive a self-supervised scheme to learn the latent warping function based on the data. The scheme makes use of pseudo-training points which are learned via optimisation. We formally introduce these notions in the methodology section.
% Finally, the measure valued warping function inherently contains a latent uncertainty and we propagate this uncertainty analytically through to prediction via the sparse spectrum Gaussian process.
In our setup, we consider the problem of learning a function $f:\InputSpace\to\R$ as a nonstationary Gaussian process. We decompose $f$ as:
\begin{equation}
f(\rawInput) = \expectation[\upperGP\circ\warp(\rawInput)|\upperGP]~,\quad \rawInput\in\InputSpace~,
\label{eq:main}
\end{equation}
where $\circ$ denotes function composition, $\upperGP:\WarpedSpace\to\R$ is a function over a latent space $\WarpedSpace$, and $\warp(\rawInput)$ represents the warped input. If $\upperGP$ has covariance function $k_\upperGP$, the resulting $f$ follows a GP with covariance $k_f(\rawInput, \rawInput') = \expectation[k_\upperGP(\warp(\rawInput), \warp(\rawInput'))]$. The latter constitutes a nonstationary kernel.

To model $\upperGP$ as a stationary GP on $\WarpedSpace$, we propose a formulation for $\warp:\InputSpace\to\WarpedSpace$, which is based on a locally affine stochastic transform:
\begin{equation}
    \warp(\rawInput) = \tfMat(\rawInput)\rawInput + \tfAdd(\rawInput)~,
\end{equation}
where $\tfMat(\rawInput)$ and $\tfAdd(\rawInput)$ are Gaussian processes. Intuitively,  the matrix $\tfMat$ scales the inputs, with a similar effect to what length-scales have on stationary kernels \cite{rasmussen2004gaussian}, but which now varies across the space, while $\tfAdd$ allows for arbitrary shifts. 

The conditional expectation \eqref{eq:main} also corresponds to the composition of a function on $\WarpedSpace$ with a measure \citep{Bauer1981} on $\WarpedSpace$, measure which is actually a function of $\rawInput\in \InputSpace$. In our case, the measure-valued warpings are Gaussian probability measures, which we parametrise as Gaussian process conditioned on \emph{pseudo-training points}. In particular, we use sparse spectrum Gaussian processes \citep{lazaro2010sparse} due to their scalability and availability of closed-form results for Gaussian inputs \citep{pan2017prediction}.

\subsection{Contributions}
\begin{itemize}
    \item We propose a new method to learn nonstationary Gaussian process models via input warping. We introduce the use of a measure-valued, self-supervised and input-dependent warping function as a natural improvement for sparse spectrum Gaussian processes. We term this \textit{sparse spectrum warped input measures} (SSWIM);
    \item We propose a self-supervised training scheme for representing the warping function allowing us to cleanly represent the latent measure valued warping; and
    \item We propose a simple extension to multiple levels of warping by propagating moments.
\end{itemize}

\section{Sparse spectrum Gaussian processes}
We start by reviewing relevant background with regards to kernel methods for Gaussian process regression. In particular, we focus on the sparse spectrum approximation to GPs \citep{lazaro2010sparse}, which we use to formulate nonstationary kernels.

\paragraph{Gaussian processes.}
Suppose our goal is to learn a function $f: \mathbb{R}^D \rightarrow \mathbb{R}$ given IID data $\mathcal{D} = \{\mathbf{x}_i, y_i\}^N_{i=1}$, with each data pair related through
\begin{equation}
    y = f(\mathbf{x}) + \epsilon, \qquad \epsilon \sim \mathcal{N}(0, \sigma_n^2),
\label{regdef}
\end{equation}
where $\epsilon$ is IID additive Gaussian noise. A Gaussian process is a distribution on functions $f$ over an input space  $\mathcal{X} \subseteq \mathbb{R}^D$ such that any finite set of inputs $\mathbf{x}_1,...,\mathbf{x}_N \in \mathcal{X}$ produces a multivariate normal distribution of response variables $\vec f_N := [f(\mathbf{x}_1),...,f(\mathbf{x}_N)]^\transpose$:
\begin{equation}
     \vec f_N \sim \mathcal{N}(\mathbf{m}_N, \mathbf{K}_N)~,
\end{equation}
where $\mathbf{m}_N = m(\mathbf{x}_1,...,\mathbf{x}_N)$ is the mean vector, and $\mathbf{K}_N =\{k(\mathbf{x}_i, \mathbf{x}_j)\}_{i,j}$ with kernel $k$.

\paragraph{Approximate GP in feature space.} \label{ssgpr}
% Considering the inner product structure $\langle \cdot,\cdot \rangle$ we can represent the kernel as $k(\mathbf{x}, \mathbf{x}') = \langle\phi(\mathbf{x}), \phi(\mathbf{x}')\rangle_\mathcal{H}$ for $\mathbf{x}, \mathbf{x}' \in \mathcal{X}$, where $\phi: \mathcal{X} \mapsto \mathcal{H}$ is a mapping from the input space $\mathcal{X}$ into potentially infinite-dimensional Hilbert space $\mathcal{H}$. Perhaps the most popular kernel is the stationary squared-exponential kernel $k(\mathbf{x}, \mathbf{x}')=\sigma^2\exp(- 2\gamma^2\|\mathbf{x}-\mathbf{x}'\|_2^2)$ with parameters $\sigma^2$ and $\gamma^2$. It is also well known that it is possible to combine standard kernels to arrive at more elaborate kernel structures \cite{duvenaud2013structure}.
Full GP inference is a challenging problem naively occurring in $\mathcal{O}(N^3)$ complexity as a consequence of having to invert an $(N,N)$ Gram matrix. 
%Much work as gone into scaling GP inference using pseudo-inputs or \textit{inducing points} in which we avoid operating on the full data space and work with a lower complexity subset of size $M$ where $M \ll N$ \cite{Hensman2013}. 
An alternative perspective on approximate GP inference is to consider the \emph{feature space} view of the kernel function using Bochner's theorem \cite{bochner1932vorlesungen}. Under this view, \emph{random Fourier features} \citep{Rahimi2008, Rahimi2009} decompose the kernel function in terms of Fourier features based on a finite approximation to the kernel's spectrum.

As presented by \citet{Rahimi2008}, the Fourier transform of any shift-invariant positive-definite kernel $k:\R^\dimension\times\R^\dimension\to\R$ yields a valid probability distribution $p_k$, so that $k$ is approximately:
\begin{equation}
k(\vec{x},\vec{x}') = \expectation_{\vec{\omega}\sim p_k}[\cos(\vec{\omega}^\transpose(\vec{x}-\vec{x}'))] \approx \vec{\phi}(\vec{x})^\transpose\vec{\phi}(\vec{x}')\,,
\end{equation}
% \begin{theorem} (Bochner's Theorem) \cite{bochner1932vorlesungen}
% The sufficient and necessary conditions for the existence of a continuous positive-definite function $\hat{\mu}: \mathbb{R}^D \to \mathbb{C}$ for all $\mathbf{x} \in \mathbb{R}^D$ is:
% \begin{equation}
% \hat{\mu}(\mathbf{x}) = \int_{\mathbb{R}^D} e^{-i\mathbf{x^\top\omegaBS}} \mathrm{d}\mu(\omegaBS)~, \quad 
% \end{equation}
% where $\mu$ is a finite non-negative Borel measure on $\mathbb{R}^D$.
% \label{thm:bochner}
% \end{theorem}

% The feature map construction in \cite{Rahimi2008} selects a known probability density function $f_\Omega(\vec{\omega})$ such that the required kernel can be reconstructed using Theorem~\autoref{thm:bochner}. The common squared exponential kernel can be reconstructed using samples drawn from a standard normal distribution. More generally, representing the known distribution using Monte-Carlo (MC) samples $\{\vec{\omega}_m \overset{\text{iid}}{\sim} f_\Omega(\vec{\omega}) \}_{m=1}^M$ defines a finite dimensional approximation of the infinite feature map $\vec{\phi}(\mathbf{x}) \in \mathbb{C}^M$ \cite{sriperumbudur2015optimal,rudi2017generalization,bach2017equivalence}. That is,
% \begin{equation}
% k(\mathbf{x}, \mathbf{x}') \approx \frac{1}{M} \sum_{m=1}^M e^{-i(\mathbf{x} - \mathbf{x}')^\top \vec{\omega}_m} =  \langle\vec{\phi}(\mathbf{x}), \vec{\phi}(\mathbf{x}')\rangle_\mathbb{C}.
% \label{eq:rff1}
% \end{equation}
where $\vec{\phi}$ corresponds to the approximate feature map:
%\begin{equation}
% $\vec{\phi}(\mathbf{x}) = \frac{1}{\sqrt{M}}[e^{-i\mathbf{x}^\top\vec{\omega}_1}, \dots ,e^{-i\mathbf{x}^\top\vec{\omega}_M}] \in \mathbb{C}^M. 
% $
% For real valued kernels, (\autoref{eq:rff1}) can be further reduced into cosine and sine terms \cite{Rahimi2008}:
\begin{equation}
\begin{split}
\vec\phi(\rawInput) = \sqrt{\frac{2}{M}} \bigg[
&\cos\left(\vec\omega_1^\transpose\rawInput\right), \dots ,\cos\left(\vec\omega_M^\transpose\rawInput \right), \sin\left(\vec\omega_1^\transpose\rawInput \right), \dots ,\sin\left(\vec\omega_M^\transpose\rawInput\right)\bigg]
\in \mathbb{R}^{2M}.
\end{split}
\label{eq:rff1}
\end{equation}
% where $\vec{\phi}(\mathbf{x}) \in \mathbb{R}^{2M}$.
% As a consequence of it's simplicity and theoretical underpinning This work has gained attention in recent years because of the simplicity, solid theoretical basis \cite{sutherland2015error,choromanski2018geometry}, and superiority in various applications \cite{rajeswaran2017towards}.
Using the feature map above we are able to define a GP termed the Sparse Spectrum Gaussian Process (SSGP) \cite{lazaro2010sparse}.
% \begin{definition} (Sparse Spectrum Gaussian Process) \cite{lazaro2010sparse}
% The Sparse Spectrum Gaussian Process (SSGP) is a GP with kernels defined on finite-dimensional and randomized feature map $\vec\phi$,
% \begin{equation}
% k(x, x') = \vec\phi(x)^\transpose\vec\phi(x') + \sigma^2_n \delta(x-x')
% \end{equation}
% where the function $\delta$ denotes the Kronecker delta.
% \label{dfn:ssgp}
% \end{definition}
% The second additive term accounts for additive zero mean noise from \eqref{regdef}.
Due to feature map \eqref{eq:rff1}, the SSGP is a Gaussian distribution over the feature weights $\vec w \in \mathbb{R}^{2M}$. If we assume the weight prior is $\mathcal{N}(\vec 0, \mat I)$, after conditioning on data $\mathcal{D}$ the posterior distribution of $\weights \sim \mathcal{N}(\vec \alpha,\, \sigma^2_n \mat A^{-1})$, where:
\begin{align} 
        % \vec w \sim \mathcal{N}(\vec \alpha,\, \sigma^2_n \mat A^{-1}),\label{ssgp1}\\
        \vec\alpha &= \mat A^{-1}\mat\Phi \vec y,\label{eq:ssgp-w-mean}\\
        \mat A &= \mat \Phi \mat \Phi^\transpose + \sigma^2_n \mat I,\label{eq:ssgp-w-mat}
\end{align}
% \begin{minipage}{.33\linewidth}
% \begin{equation}
% \vec w \sim \mathcal{N}(\vec \alpha,\, \sigma^2_n \mat A^{-1}),\label{ssgp1}
% \end{equation}
% \end{minipage}%
% \begin{minipage}{.33\linewidth}
% \begin{equation}
% \vec\alpha = \mat A^{-1}\mat\Phi \vec y,\label{ssgp2}
% \end{equation}
% \end{minipage}%
% \begin{minipage}{.33\linewidth}
% \begin{equation}
%  \mat A = \mat \Phi \mat \Phi^\transpose + \sigma^2_n \mat I,\label{ssgp3}
% \end{equation}
% \end{minipage}
following from Bayesian Linear Regression \cite{bishop2007pattern}.
The design matrix $\mat\Phi = [\vec\phi(\mathbf{x}_1), ...., \vec\phi(\mathbf{x}_N)]$ and column vector $\vec y = [y_1,...,y_N]^\transpose$ are given directly by the data $\mathcal{D}$. The posterior distribution over the response $y$ given an $\mathbf{x}$ is exactly Gaussian:
\begin{equation} \label{ssgppred}
	p(f(\vec x)|\vec x) = \mathcal{N}\left(\vec \alpha^\transpose\vec\phi(x), \sigma^2_n \norm{\vec\phi(x)}^2_{\mat A^{-1}}\right),
\end{equation}
where we define $\norm{\vec v}^2_{\mat \Sigma} := \vec v^\transpose \mat \Sigma \vec v$. 
Multivariate outputs can be modelled as conditionally independent GPs for each output dimension or jointly by encoding the covariance between the outputs as a vector-valued GP \cite{Alvarez2011}.

\paragraph{Kernels with Gaussian inputs.}
\label{sec:uncertain-pred}
% Input warping methods for nonstationarity require some \textit{functional form} of the warping. For example in \cite{wilson2016deep} the warping is a deterministic neural network, in \cite{snoek2014input} it is a deterministic monotonic function, and in \cite{hegde2019deep} the warping is defined through a stochastic differential equation. In contrast, with our method, we propose to explicitly learn an operator-valued input-dependent function $\tfMat$ that combines with the original input data before being passed as uncertain inputs into an SSGP. The only difference is that we also wish to propagate any uncertainty on the warping function into the top-level function $\upperGP$ which will produce the final predictions taking that underlying uncertainty into account.
In our formulation of non-stationary kernel, we take form the kernel based on expectations with respect to distributions conditioned on the inputs. In the sparse-spectrum formulation, the expected kernel is simply the result of the inner product between the expected feature map of each input, due to the linearity of expectations. For the case of Gaussian inputs, results by \citet{pan2017prediction} then allow us to compute the expected feature map in closed form. In the cosine case, we have:
\begin{equation}
\mathbb{E}[\cos(\vec{\omega}^\transpose\vec{\tilde{x}})] = \exp\left(-\frac{1}{2}\norm{\vec{\omega}}^2_{\mat\Sigma}\right)\cos(\vec{\omega}^\transpose\vec{\hat{x}})~.
% \mathbb{E}[\sin(\vec{\omega}^\transpose\vec{\tilde{x}})] &= \exp\left(-\frac{1}{2}\norm{\vec{\omega}}^2_{\mat\Sigma}\right)\sin(\vec{\omega}^\transpose\vec{\hat{x}}).
\label{eq:trigexp}
\end{equation}
For the sine components of $\vec\phi$, the result is similar, only swapping cosines by sines above. What is important to note here is the exponential constant out the front which scales the standard feature by a value proportional to the uncertainty in the feature map's kernel frequencies $\vec{\omega}$. That is to say, expectations that take on larger (predictive) uncertainties will be \textit{smaller} than if we did not take this uncertainty into account.

\section{Sparse spectrum warped input measures}
In this section we introduce the main contribution of the paper: \emph{sparse spectrum warped input measures} (SSWIM). The key idea in our work is based on two crucial steps. First, we construct a stochastic vector-valued mapping modelling the input warping $\warp:\InputSpace\to\WarpedSpace$, where $\InputSpace\subseteq\R^\dimension$ represents the raw input space and $\WarpedSpace$ is the resulting warped space. 
% The warped inputs are modelled as $\warp(\rawInput) = \tfMat(\rawInput)\rawInput + \tfAdd(\rawInput)$ by a linear operator-valued GP, $\tfMat(\rawInput)\in\LinearOpSpace(\InputSpace)$, so that $\WarpedSpace\subseteq\InputSpace$. 
A top-level GP modelling $\upperGP:\WarpedSpace\to\R$ then estimates the output of the regression function $f:\InputSpace\to\R$. To learn the warping, each lower-level SSGP is provided with \textit{pseudo-training} points, 
%$\{ \mathbf{X}_{\tfMat} , \mathbf{Y}_{\tfMat}\}$,
which are learned jointly with the remaining hyper-parameters of both GP models.

% \begin{remark}
It is important to note that the pseudo-training points are \textit{free parameters} of the latent warping function and therefore \textit{hyperparameters of the top-level function}. Furthermore, while our construction and implementation assumes a pseudo-training dimensionality equal to that of the original data $\dim{\InputSpace} = \dim{\WarpedSpace}$, nothing preventing us from embedding the input warping into a lower $\dim \WarpedSpace \ll \dim \InputSpace$ or higher $\dim \WarpedSpace \gg \dim \InputSpace$ dimensional manifold. 
% \end{remark}

\subsection{Warped input measures}
To model and propagate the uncertainty on the warping operator $\tfMat$ through the predictions, we start by modelling $\tfMat:\InputSpace\to\LinearOpSpace(\InputSpace)$ as a Gaussian process. Then every linear operation on $\tfMat$ results in another GP \cite{Jidling2017}, so that $\warp(\vec{x}) = \tfMat(\vec{x})\vec{x} + \tfAdd(\rawInput)$, for a deterministic $\vec{x}\in\InputSpace$, is Gaussian. Similarly, as expectations constitute linear operations, the \emph{expected value} of the GP $\upperGP$ under the random input given by the warping is also Gaussian \cite{Oliveira2019}. Marginalising $\warp$ out of the predictions, i.e. inferring the expected value of $f$ under the distribution of $\warp$, $\hat{f}(\vec{x}) = \expectation[\upperGP\circ\warp(\vec{x})|\upperGP]$, we end up with a final GP, which has analytic solutions.

% We do this by approximating the true predictive distribution $p(y)$ by a Gaussian distribution with analytically computed moments \cite{pan2017prediction}.

From \autoref{sec:uncertain-pred}, the uncertain-inputs predictions from $\hat{\upperGP} = \expectation[\upperGP(\randomInput)|\upperGP]$ for $\warp(\rawInput) \sim \normal(\vecMean{\warp}(\rawInput), \mat\Sigma_\warp(\rawInput))$ are given by the SSGP predictive equations in $\eqref{ssgppred}$ using the expected feature map for $\expectation[\vec\phi(\warp(\rawInput))]$. \autoref{eq:trigexp} then allows us to compute $\expectation[\vec\phi(\warp(\rawInput))]$ in closed form for a given mean $\vecMean{\warp}(\rawInput)$ and covariance matrix $\mat\Sigma_\warp(\rawInput)$.
% \begin{equation}
% \begin{split}
%     \mathbb{E}[u(\vec{\tilde{x}})]
%     &= \mathbb{E}[\mathbb{E}[\upperGP(\vec{\tilde{x}})|\vec{\tilde{x}}]]\\
%     &= \mathbb{E}[\vec{\alpha}^\transpose\vec{\phi}(\vec{\tilde{x}})]\\ 
%     &= \vec{\alpha}^\transpose\mathbb{E}\begin{bmatrix}
%       \vec{\phi}^c\\
%       \vec{\phi}^s\\
%      \end{bmatrix},
% \end{split}
% \end{equation}
% where $\mathbb{E}[\vec\phi^c_i] = \sigma_k\mathbb{E}[\cos(\vec{\omega}^\transpose_i \mathbf{x})]$ and $\mathbb{E}[\vec\phi^s_i] = \sigma_k\mathbb{E}\sin(\vec{\omega}^\transpose_i \mathbf{x})$, and $i = 1,...,M$.
% In addition, we have the following result by \cite{pan2017prediction}
% Now we consider the variance of the expected process, i.e. the variance of the stochastic process defining the expected value of the original GP under the uncertain warped input $\warp(\vec{x})$. Using the expected feature map, the predictive variance adopts the same form of the original one in \eqref{ssgppred}, simply replacing $\vec\phi(\vec{x})$ by $\expectation[\vec\phi(\vec{\tilde{x}})]$. In our case, we have $\vec{\tilde{x}}:=\warp(\vec{x})=\tfMat(\vec{x})\vec{x}$. 
The general form of the covariance matrix $\mat\Sigma_\warp(\rawInput)$
for $\warp(\vec{x})$ involves dealing with a fourth order tensor describing the second moment of $\tfMat$. For this paper, however, we consider a particular case with a more elegant formulation and yet flexible enough to accommodate for a large variety of warpings.

Let $\tfMat(\vec{x})\vec{x} := \tfVec(\vec{x})\odot\vec{x}$, where $\odot$ denotes the element-wise product and $\tfVec$ is a vector-valued Gaussian process. This type of warping is equivalent to $\tfMat(\rawInput)$ map to a diagonal matrix. The mean and covariance matrix of the warped input $\warp(\vec{x})$, can be derived as (see Appendix for details):
% Applying \autoref{thr:warp-uncetainty} to $\warp(\rawInput) := \tfVec(\rawInput)\odot\rawInput + \tfAdd(\rawInput)$ results in $\warp(\rawInput) \sim \normal(\vecMean{\warp}(\rawInput), \mat\Sigma_\warp(\rawInput))$, which is parametrised by:
\begin{align}
    \vecMean{\warp}(\rawInput) &= \vecMean{\tfVec}(\rawInput) \odot \rawInput + \vecMean{\tfAdd}(\rawInput)\label{eq:warp-mean}\\
    \mat\Sigma_\warp(\rawInput) &= \rawInput\vec 1^\transpose \odot \mat\Sigma_\tfVec(\rawInput) \odot \vec 1 \rawInput^\transpose + \mat\Sigma_\tfAdd(\rawInput)~,\label{eq:warp-cov}
\end{align}
where $\vecMean{\tfVec}(\rawInput)$ and $\mat\Sigma_{\tfVec}(\rawInput)$ are the predictive mean and covariance, respectively, of the GP defining $\tfVec$, while $\vecMean{\tfAdd}(\rawInput)$ and $\mat\Sigma_{\tfAdd}(\rawInput)$ are the same for the GP on $\tfAdd$.

% Defining $\vecMean{\tfVec}_\rawInput:=\expectation[\tfVec_\rawInput]$ and $\tfVec_\rawInput:=\tfVec(\rawInput)$, the covariance matrix of $\tfVec(\vec{x})\odot\vec{x}$ is given by:
% \begin{equation}
% \begin{split}
%     \variance[\tfVec_\rawInput\odot\rawInput] &= \expectation[((\tfVec_\rawInput-\hat{\tfVec}_\rawInput)\odot\rawInput)((\tfVec_\rawInput-\hat{\tfVec}_\rawInput)\odot\rawInput)^\transpose]\\
%     &=\rawInput\odot\expectation[(\tfVec_\rawInput-\hat{\tfVec}_\rawInput)(\tfVec_\rawInput-\hat{\tfVec}_\rawInput)^\transpose]\odot \rawInput^\transpose\\
%     &= \rawInput \odot \variance[\tfVec_\rawInput]\odot\rawInput^\transpose~,
% \end{split}
% \end{equation}
% where we apply the symmetry of the element-wise product, i.e. $\rawInput\odot\vec{a} = \vec{a}\odot\rawInput$ and allow the element-wise product to broadcast vectors through matrix rows and columns.

% \subsection{Learning the warping operator}

\begin{algorithm}[t]
\caption{Sparse Spectrum Warped Input Measures}
\begin{algorithmic}
% \SetAlgoLined
% \SetKwInOut{Input}{Input} % Note: the second one is the displayed text
% \SetKwInOut{Output}{Output}

\STATE \textbf{Input:} $\{ \mathbf{X}, \mathbf{y}\}$
\STATE \textbf{Output:} $\vec\theta = \{\vec\theta_{\upperGP}, \vec\theta_{\tfVec}, \vec\theta_\tfAdd, \mat{X}_{\tfVec}, \mat{Y}_{\tfVec}, \mat{X}_{\tfAdd}, \mat{Y}_{\tfAdd} \}$ 
\STATE Initialize pseudo-training points  $\{ \mat{X}_{\tfVec}, \mat{Y}_{\tfVec} \}$, $\{ \mat{X}_{\tfAdd}, \mat{Y}_{\tfAdd} \}$
% $\mat{X}_{\tfVec} \sim \mathcal{U}(\min(\mathbf{X}), \max(\mathbf{X}))$
% \STATE Initialize pseudo-targets $\mat{Y}_{\tfVec} \sim \mathcal{N}(1, \sigma_{\tfVec}^2)$
\FOR{$t \in \{1,\dots,T\}$}
  \STATE Fit $\tfVec$ and $\tfAdd$ to $\{ \mat{X}_{\tfVec}, \mat{Y}_{\tfVec} \}$, $\{ \mat{X}_{\tfAdd}, \mat{Y}_{\tfAdd} \}$
  \STATE Compute $\vecMean{\warp}$ and $\mat\Sigma_\warp$ for $\mat X$
%   \STATE Construct input-warped expected feature map $\mat{\hat\Phi}$
  \STATE Fit $\upperGP$ using expected feature map
  \STATE Calculate $\log p(\mathbf{y} | \vec{\theta})$
  \STATE Update gradients and take new step.
\ENDFOR
\end{algorithmic}
\end{algorithm}

\subsection{Latent self-supervision with pseudo-training}
In order to fully specify our latent function, we utilise \textit{pseudo-training} pairs \{$\mat{X}_{\tfVec}, \mat{Y}_{\tfVec}$\} and $\{\mat{X}_{\tfAdd}, \mat{Y}_{\tfAdd}\}$, somewhat analogous to the well known \textit{inducing-points} framework for sparse Gaussian processes \cite{titsias2009variational}. Conditioning on these virtual observations allows us to implicitly control the Gaussian measure determined by the warping SSGP. % This process is a form of supervised regression of the latent warping function with pseudo-training data. 
% As it is an approximate GP, this function $\tfMat$ similarly has its own kernel $k_\tfMat$ and hyper-parameters $\vec{\theta}_{\mathbf{G}}$, which are jointly optimized during learning.

% Talk about initialisation 
% Talk about dimensional complexity and coverage
% Ta

% \subsection{Algorithmic complexity}
% % 1. Discuss improvement over full GP.
% % 2. Explain complexity of standard BLR. 
% % 3. Explain complexity of having the multi-output latent function

We model the multiplicative warping $\tfVec: \mathbb{R}^{D} \rightarrow \mathbb{R}^{D}$ using a standard, multi-output, SSGP that is analytically fit on virtual \textit{pseudo-training} points $\{\mat{X}_{\tfVec}, \mat{Y}_{\tfVec} \}$. Assuming coordinate-wise output independence, we model $\tfVec$ as 
$\tfVec(\rawInput) \sim \normal(\vecMean{\tfVec}(\rawInput), \mat\Sigma_\tfVec(\rawInput))$, where:
\begin{align}
    \vecMean{\tfVec}(\rawInput) &= \vec{\phi}_\tfVec(\rawInput)^\transpose \mat A_\tfVec^{-1}\mat\Phi_\tfVec\mat Y_\tfVec\label{eq:g-mean}\\
    \mat\Sigma_\tfVec(\rawInput) &= \sigma_{n,g}^2 \vec{\phi}_\tfVec(\rawInput)^\transpose \mat A_\tfVec^{-1}\vec{\phi}_\tfVec(\rawInput)\mat I~, \label{eq:g-cov}
\end{align}
with $\mat\Phi_\tfVec := \mat\Phi_\tfVec(\mat X_\tfVec)$ as the matrix of Fourier features for the pseudo-inputs $\mat X_\tfVec$, and $\mat A_\tfVec = \mat\Phi_\tfVec\mat\Phi_\tfVec^\transpose + \sigma_{n,g}^2\mat I$. The pseudo-inputs $\mat{X}_{\tfVec}$ are initially sampled uniformly across the data domain, $\mat{X}_{\tfVec} \sim \mathcal{U}(\min(\mathbf{X}), \max(\mathbf{X}))$. The pseudo-training targets $\mat{Y}_{\tfVec}$ are initialised $[\mat{Y}_{\tfVec}]_{i,j} \sim \mathcal{N}(1, \sigma_\gamma^2)$ where $\sigma_\gamma^2$ mimics a prior variance for the latent warping function. The mean at $1$ keeps the initial warping close to identity. 

We adopt a similar construction for the GP on the additive component of the warping $\tfAdd$. However, we initialise the pseudo-training targets $\mat{Y}_\tfAdd$ with zero-mean values $[\mat{Y}_\tfAdd]_{i,j} \sim \normal(0,\sigma^2_\gamma)$, so that we favour a null effect initially. In summary, the complete expected kernel is thus given as:
\begin{equation}
k_f(\rawInput, \rawInput') := \expectation[\ffeature(\vec{\warp(\rawInput)}]^\transpose\expectation[\ffeature(\vec{\warp(\rawInput')})]~,
% \expectation[\cos(\vec{\omega}^\transpose(\vec{\warp(\rawInput)}-\vec{\warp(\rawInput')}))] \approx \expectation[\vec{\phi}(\vec{\warp(\rawInput)})]^\transpose\expectation[\vec{\phi}(\vec{\warp(\rawInput')})]\,,
\end{equation}
% Using the SSGP equations, we fit the latent model to these pseudo-training points.
where the expectation is taken over $\warp$, whose distribution is recursively defined by equations \ref{eq:warp-mean} to \ref{eq:g-cov}.
% Our latent function $\mathbf{G}$ is now capable of producing a predictive \textit{distribution} as per \eqref{ssgppred}. 

\subsection{A layered warping}
We have thus far considered a single warping $\warp$ of the input $\rawInput$. It is natural to ask: \emph{can we warp the warpings?} A simple way to answer this is to revisit how we interpret a single warping: we are transforming the original input space, with which our response varies in a non-stationary way, to a new space a GP with a stationary kernel can easily represent. We could thus intuit a warping of a warping to mean that we are transforming the first level of warping to a second one to which our response variable is simply \emph{more stationary} than if we had just relied on the first warping alone. We present now an extension to SSWIM which lets us perform this measure value warping of a measure valued warping. Let us begin by defining the $J^{\text{th}}$ warping as:
\begin{equation}
\warp^{(J)}(\rawInput^{(J-1)}) = \tfVec^{(J)}(\rawInput^{(J-1)})\odot\rawInput^{(J-1)} + \tfAdd^{(J)}(\rawInput^{(J-1)}),
\label{eq:layered}
\end{equation}
where:
\begin{align}
    \rawInput^{(J-1)} &= \warp^{(J-1)}(\rawInput^{(J-2)}) \quad,~J\geq 2 %\\
    % \hat{\tfMat}^{(J)}(\rawInput^{(J-1)}) &= \expectation_{\rawInput^{(J-1)}}[\tfMat(\rawInput^{(J-1)})|\tfMat]\\
    % \hat{\tfAdd}^{(J)}(\rawInput^{(J-1)}) &=\expectation_{\rawInput^{(J-1)}}[\tfAdd(\rawInput^{(J-1)})|\tfAdd]
\end{align}
While multiplication of a known vector by a Gaussian random matrix keeps Gaussianity, after the first warping layer, the product of two Gaussians is no longer Gaussian in \eqref{eq:layered}. For the layered formulation, we therefore apply moment matching to approximate each layer's warped input as a Gaussian $\rawInput^{(J)} \sim \normal(\vecMean\rawInput^{(J)}, \mat\Sigma^{(J)}_\rawInput)$. Making independence assumptions on \eqref{eq:layered} and applying known results for the Hadamard product of independent random variables \cite{Neudecker1995},  we have:
\begin{align}
    \vecMean\rawInput^{(J)} &= \vecMean\tfVec^{(J)}\odot\vecMean\rawInput^{(J-1)} + \vecMean\tfAdd^{(J)}\\
    \begin{split}
    \mat\Sigma^{(J)}_\rawInput &= \mat\Sigma^{(J-1)}_\rawInput\odot\mat\Sigma^{(J)}_\tfVec + \mat\Sigma^{(J-1)}_\rawInput\odot\vecMean\tfVec^{(J)}\vecMean\tfVec^{(J)\transpose}+\mat\Sigma^{(J)}_\tfVec\odot\vecMean\rawInput^{(J-1)}\vecMean\rawInput^{(J-1)\transpose} + \mat\Sigma_\tfAdd^{(J)}~,
    \end{split}
\end{align}
where $\tfVec^{(J)} \sim \normal(\vecMean{\tfVec}^{(J)}, \mat\Sigma^{(J)}_\tfVec)$ and $\tfAdd^{(J)} \sim \normal( \vecMean\tfAdd^{(J)}, \mat\Sigma_\tfAdd^{(J)})$ are the SSGP predictions using the expected feature map (\autoref{eq:trigexp}) of the previous layer's output $\rawInput^{(J-1)}$.

\subsection{Joint training}
The goal of optimization in learning our warping with uncertainty is to quickly discover hyper-parameters whose models explain the variation in the data. We also want to avoid pathologies that may manifest with an arbitrarily complex warping function. We do this by minimising the model's negative log-marginal likelihood. Given a set of observations $\set{D} = \{\rawInput_i, \observation_i\}_{i=1}^\nObs$, we learn the hyper-parameters $\vec{\theta}$ by minimising the negative log-marginal likelihood:
\begin{equation}
    -\log p(\vec{y}|\vec{\theta}) = \frac{1}{2\sigma^2_n}(\observations^\transpose\observations - \observations^\transpose \mat{\hat{\Phi}}_F^\transpose\mat A_F^{-1}\mat{\hat{\Phi}}_F\observations) + \frac{1}{2}\log\lvert\mat A_F\rvert - \frac{\dimension}{2}\log\sigma_n^2 + \frac{\nFeatures}{2}\log 2 \pi \sigma_n^2~,
\end{equation}
where $\mat{\hat{\Phi}}_F$ is the matrix whose rows to expected feature maps for the top-level SSGP, i.e. $[\mat{\hat{\Phi}}_F]_i = \expectation_{\warp}[\vec\phi_F(\warp(\rawInput_i))]^\transpose$, and $\lvert\mat A_F\rvert$ denotes the determinant of $\mat A_F$. The expectation is taken under the warping $\warp$, whose parameters are computed from the predictive mean and covariance functions of the lower-level GPs (cf. \eqref{eq:warp-mean} and \eqref{eq:warp-cov}), and available in closed form via \autoref{eq:trigexp}.

% Recall our top-level uncertain inputs predictive process $F$ inherits uncertain inputs $\mathbf{G}(\mathbf{X})\mathbf{X}$
% Mention with reference to the experimental section where we compare learning without the uncertainty and learning with.
% TODO See the Moriconi paper. 
% Specify the optimisation objective noting the augmented LML

\begin{figure}[t]
    % \vspace{-10pt}
\centering
\includegraphics[width=1.00\linewidth]{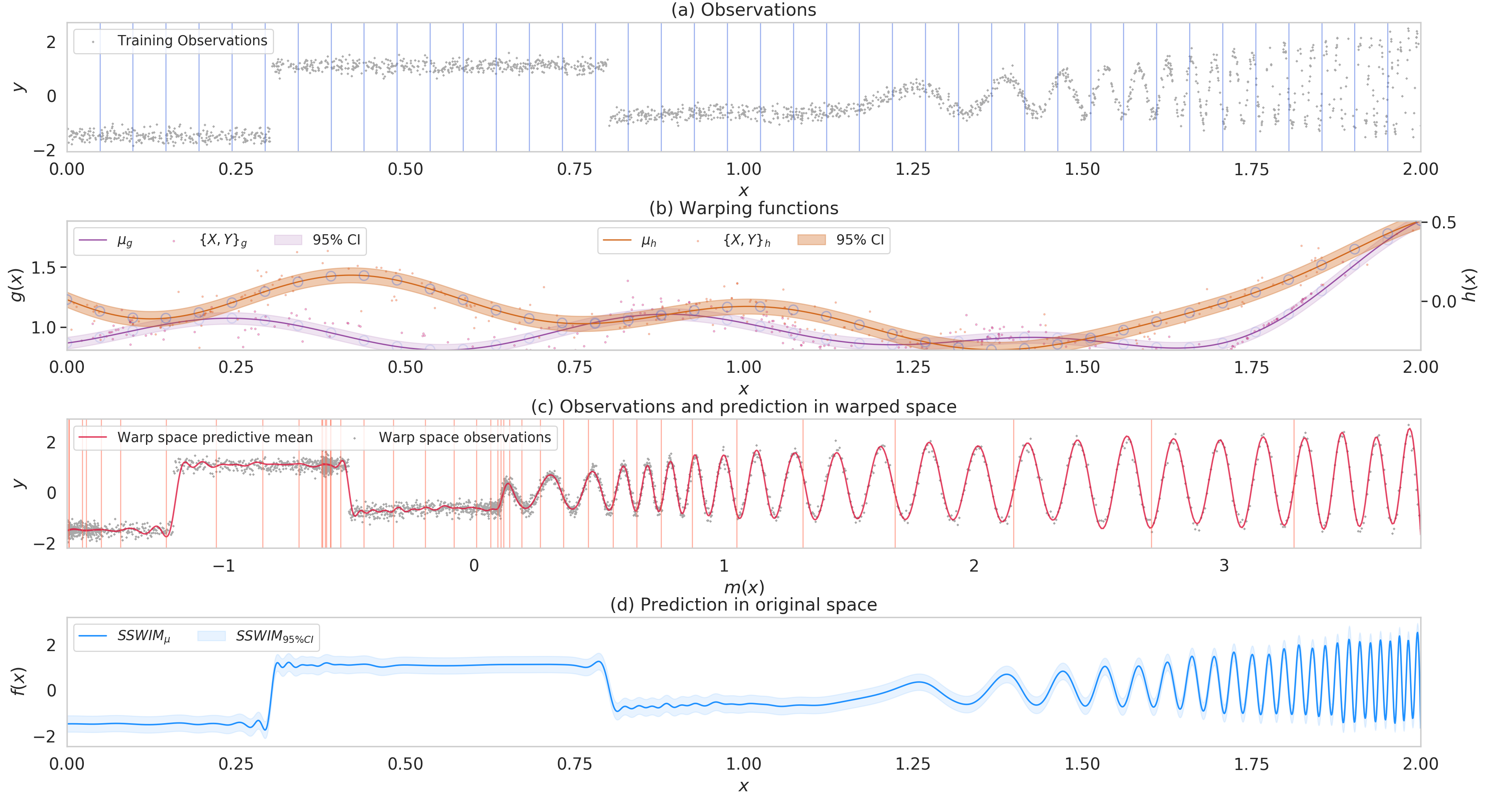}
\caption{Visualisation of SSWIM learning an input warping. (a) Noisy training data. Going left to right, the signal observations exhibit abrupt steps, periodic, and spatial frequency nonstationarity. (b) The learned warping functions. (c) The training data after input warping, and (d) Final prediction with respect to the warped inputs. The key observation here is how the spatially varying frequencies and steps in the original training data from (a) have been transformed to (c) where the warped data varies in a more uniform (stationary) manner. \label{fig:sswim_vis1}}
%   \vspace{-5pt}
\end{figure}

\begin{figure*}[!ht]
    % \vspace{-3pt}
  \centering
  \includegraphics[width=\linewidth]{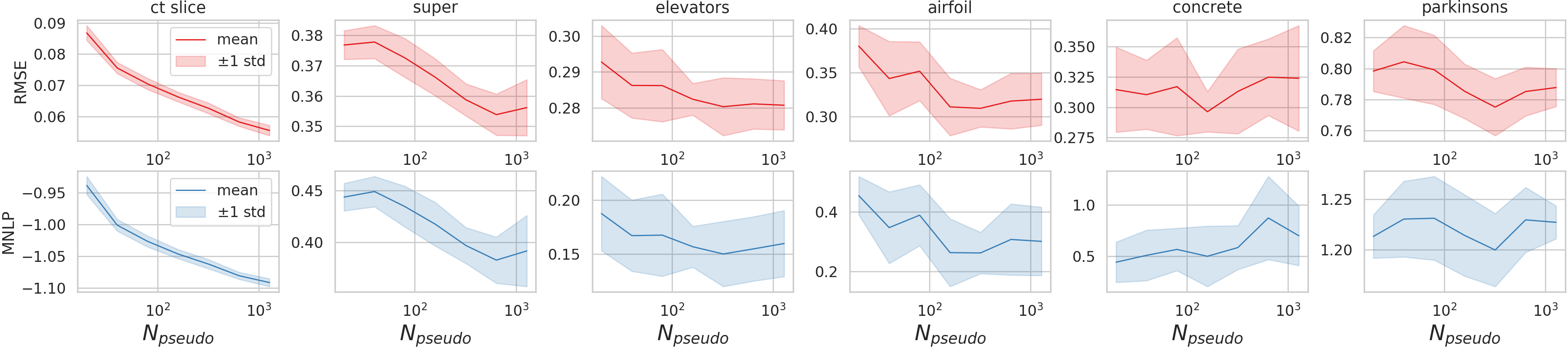}
  \caption{Performance in RMSE and MNLP as the number of pseudo-training points increases.\label{fig:perf_v_pseudo}}
  \vspace{-5pt}
\end{figure*}

\begin{table*}[ht]
% \tiny 
\scriptsize
\centering
\caption{RMSE and MNLP metrics for various real world datasets.}
% \begin{tabular}{c|l|llllllllll}
\begingroup
\setlength{\tabcolsep}{2.75pt} % Default value: 6pt
\renewcommand{\arraystretch}{1.0} % Default value: 1
% \begin{tabular}{l|l|llllllllll}
\begin{tabular}{l|l|cccccccccc}
\toprule
 \multicolumn{1}{c}{$\quad$}&\multicolumn{1}{c}{\textit{\tiny{$(D, N)$}}}          &\multicolumn{1}{c}{\tiny{$(8, 1030)$}}         &\multicolumn{1}{c}{\tiny{$(16, 5875)$}}       &\multicolumn{1}{c}{\tiny{$(15, 17379)$}}         &\multicolumn{1}{c}{\tiny{$(379, 53500)$}}          &\multicolumn{1}{c}{\tiny{$(81,21263)$}}       &\multicolumn{1}{c}{\tiny{$(9,45730)$}}          &\multicolumn{1}{c}{\tiny{$(77,583250)$}}             &\multicolumn{1}{c}{\tiny{$(90,515345)$}}              \\
 \multicolumn{1}{c}{$\quad$}&\multicolumn{1}{c}{Method}            &\multicolumn{1}{c}{concrete}         &\multicolumn{1}{c}{parkinsons}       &\multicolumn{1}{c}{bikeshare}         &\multicolumn{1}{c}{ct slice}          &\multicolumn{1}{c}{supercond}       &\multicolumn{1}{c}{protein}          &\multicolumn{1}{c}{buzz}             &\multicolumn{1}{c}{song}              \\
 \hline\parbox[t]{1mm}{\multirow{8}{*}{\rotatebox[origin=c]{90}{RMSE ($\times 10^{-1}$)}}}  
 &SSWIM$_1$  & 3.05 $\pm$ 0.26  & 7.63 $\pm$ 0.20  & 0.13 $\pm$ 0.04   & 0.46 $\pm$ 0.02   & 3.44 $\pm$ 0.14 & 5.91 $\pm$ 0.07  & 2.98 $\pm$ 0.04  & 8.12 $\pm$ 0.05   \\
 &SSWIM$_2$ &\textbf{ 3.01 $\pm$ 0.31 } & \textbf{7.55 $\pm$ 0.15}  & 0.11 $\pm$ 0.03   & \textbf{0.23 $\pm$ 0.01}   & \textbf{3.02 $\pm$ 0.04} & \textbf{5.80 $\pm$ 0.08}  & \textbf{2.40 $\pm$ 0.01}  & \textbf{7.97 $\pm$ 0.03}   \\
 &DSDGP      & 5.88 $\pm$ 1.24  & 7.94 $\pm$ 0.20  & 0.33 $\pm$ 0.55   & 4.81 $\pm$ 1.18   & 5.10 $\pm$ 0.84 & 5.96 $\pm$ 0.06  & 3.65 $\pm$ 0.75  & 8.46 $\pm$ 0.03   \\
 &DKL       & 3.18 $\pm$ 0.38  & 8.84 $\pm$ 0.74  & 0.24 $\pm$ 0.03   & 0.52 $\pm$ 0.08   & 3.46 $\pm$ 0.18 & 7.15 $\pm$ 1.10  & 4.11 $\pm$ 3.33  & 16.66 $\pm$ 8.14  \\
 &RFFNS      & 3.46 $\pm$ 0.24  & 8.15 $\pm$ 0.15  & 0.05 $\pm$ 0.01  & 4.39 $\pm$ 0.27   & 3.85 $\pm$ 0.05 & 6.87 $\pm$ 0.06  & 5.70 $\pm$ 0.84  & 8.35 $\pm$ 0.03   \\
 &SVGP       & 3.32 $\pm$ 0.26  & 8.14 $\pm$ 0.12  & 0.06 $\pm$ 0.03   & 1.16 $\pm$ 0.02   & 4.06 $\pm$ 0.05 & 7.32 $\pm$ 0.08  & 9.98 $\pm$ 0.02  & 12.19 $\pm$ 0.18  \\
 &SGPR      & 5.55 $\pm$ 0.58  & 7.86 $\pm$ 0.22  & 0.67 $\pm$ 0.18   & 1.79 $\pm$ 0.04   & 4.27 $\pm$ 0.06 & 6.45 $\pm$ 0.07  & 2.89 $\pm$ 0.02  & 8.40 $\pm$ 0.04   \\
 &RFFS       & 3.33 $\pm$ 0.30  & 8.24 $\pm$ 0.17  & \textbf{0.03 $\pm$ 0.00}   & 2.34 $\pm$ 0.05   & 3.89 $\pm$ 0.06 & 6.91 $\pm$ 0.07  & 3.78 $\pm$ 0.14  & 8.36 $\pm$ 0.04   \\
 \midrule \parbox[t]{1mm}{\multirow{8}{*}{\rotatebox[origin=c]{90}{MNLP ($\times 10^{-1}$)}}} 
 &SSWIM$_1$ & 10.22 $\pm$ 4.15 & 11.95 $\pm$ 0.47 & -11.89 $\pm$ 0.15 & -11.24 $\pm$ 0.05 & 3.55 $\pm$ 0.32 & 8.95 $\pm$ 0.12  & 2.03 $\pm$ 0.13  & 12.08 $\pm$ 0.05  \\
 &SSWIM$_2$ & 5.19 $\pm$ 2.59  & 12.50 $\pm$ 0.44 & -11.78 $\pm$ 0.07 & \textbf{-11.79 $\pm$ 0.02} & \textbf{2.82 $\pm$ 0.29} & \textbf{8.82 $\pm$ 0.15}  & \textbf{-0.09 $\pm$ 0.04} & \textbf{11.93 $\pm$ 0.04}  \\
 &DSDGP     & 11.02 $\pm$ 1.06 & \textbf{11.91 $\pm$ 0.24} & -23.28 $\pm$ 8.29 & 6.62 $\pm$ 2.61   & 7.36 $\pm$ 1.62 & 9.04 $\pm$ 0.10  & 3.80 $\pm$ 2.02  & 12.52 $\pm$ 0.04  \\
 &DKL        & 7.69 $\pm$ 0.20  & 13.17 $\pm$ 1.12 & 6.82 $\pm$ 0.01   & 6.83 $\pm$ 0.01   & 7.76 $\pm$ 0.10 & 11.02 $\pm$ 1.53 & 9.01 $\pm$ 4.65  & 42.64 $\pm$ 44.77 \\
 &RFFNS     & 3.31 $\pm$ 0.45  & 12.18 $\pm$ 0.18 & -11.97 $\pm$ 0.00 & 5.95 $\pm$ 0.66   & 4.66 $\pm$ 0.12 & 10.39 $\pm$ 0.08 & 8.78 $\pm$ 1.87  & 12.39 $\pm$ 0.04  \\
 &SVGP       & \textbf{2.83 $\pm$ 0.56}  & 12.21 $\pm$ 0.14 & \textbf{-27.70 $\pm$ 1.24} & -5.98 $\pm$ 0.13  & 5.32 $\pm$ 0.12 & 11.10 $\pm$ 0.09 & 63.31 $\pm$ 3.44 & 18.02 $\pm$ 0.09  \\
 &SGPR      & 8.48 $\pm$ 2.10  & 12.39 $\pm$ 0.30 & -13.67 $\pm$ 0.98 & -3.14 $\pm$ 0.26  & 5.58 $\pm$ 0.10 & 10.05 $\pm$ 0.13 & 1.11 $\pm$ 0.12  & 11.97 $\pm$ 0.07  \\
 &RFFS      & 3.05 $\pm$ 0.96  & 12.29 $\pm$ 0.20 & -11.98 $\pm$ 0.00 & -0.33 $\pm$ 0.22  & 4.79 $\pm$ 0.13 & 10.45 $\pm$ 0.09 & 4.41 $\pm$ 0.37  & 12.40 $\pm$ 0.05  \\
\bottomrule
\end{tabular}
\endgroup
\label{table:experiments_main}
\end{table*}

\section{EXPERIMENTS}
We experimentally validate SSWIM alongside various state of the art methods in both small and large data regimes as well as expand upon the intuition in \autoref{sec:analysis1interpret} by examining specific aspects of the model. Section \ref{sec:analysis2pseudo} analyses computational complexity and model performance with respect to the pseudo-training points. We investigate increasing the number of warping levels in \autoref{sec:analysis3layers}. The large scale comparison alongside various algorithms is presented in \autoref{sec:experiments_main}.

For every quantitative experiment, we report the mean and standard deviation over ten repeats. Metrics are presented in the standardized data scale. In all experiments the Matern $\frac{3}{2}$ is used as the base kernel. For performance evaluation we use the test Root Mean Square Error (RMSE) and test Mean Negative Log Probability (MNLP). These are defined as RMSE = $\sqrt{\langle (y_{*j} - \mu_{*j})^2  \rangle}$ and MNLP = $\frac{1}{2}\langle (\frac{y_{*j} - \mu_{*j}}{\sigma_{*j}})^2 + \log\sigma^2_{*j} + \log2\pi \rangle$ where  $y_{*j}$ is the true test value, and $\mu_{*j}$ and $\sigma^2_{*j}$ are the predictive mean and variance respectively for the $j^{\text{th}}$ test observation. Mean is denoted as $\langle \cdot \rangle$. 

\subsection{Analysis}\label{sec:analysis1all}
% TODO DISCUSS FIGURE (dont forget to refernce it!)
\subsubsection{Inductive bias and a geometric interpretation}\label{sec:analysis1interpret}

An intuitive interpretation of SSWIM is by imagining it as learning a conditional affine transformation. The quintessential affine transformation of some vector $x$ is described as $\mathbf{A}x + b$ for some multiplication matrix $\mathbf{A}$ and addition vector $b$. Such transformations are typically interpreted geometrically \cite{gonzalez2008digital} as \textit{translation}, \textit{rotation}, \textit{reflection}, \textit{scale} and \textit{shear}. SSWIM learns a \textit{conditional} affine map that \textit{depends on the input}. I.e. $\mathbf{A}$ and $b$ become maps $\mathbf{A}(x)$ and $b(x)$. By directly applying a learned warping to the original input data we transform the inputs into a locally Euclidean manifold which ultimately preserves any structure with respect to the input resulting in a convenient inductive bias. Observe in \autoref{fig:sswim_vis1} (c) how we have non-uniformly "stretched out" out the left and rightmost parts of the original data in (a) to produce a new warped dataset. What was original spatially nonstationary becomes spatially homogeneous resulting excellent prediction as in \autoref{fig:sswim_vis1} (d).

\subsubsection{How many pseudo-training points?}\label{sec:analysis2pseudo}
% TODO discuss the relation between density of the pseudo-training points and their initial hyperparameters. Future work 
To understand the overall sensitivity of our method we visualise the predictive performance as a function of the number of pseudo-training points. \autoref{fig:perf_v_pseudo} shows performance, in log scale, with respect to the number of pseudo-training points on real world datasets. We observe trending improvement however very few pseudo-targets are required to get excellent performance, even in much higher dimensional problems like  \textit{superconductivity} ($D=81$) and \textit{ct slice} ($D=379$), suggesting that there is significant expressiveness in the underlying warping function. 

We remark that a possible drawback of pseudo-training points and fitting a stochastic model over those points is the question of how to set the prior of their locations. Furthermore, how do we initialise them? To answer this, it is natural to set $\tfMat$ and $\tfAdd$ to be fit to noisy pseudo-targets with mean $\mathbf{I}$ and $\mathbf{0}$ respectively. This has a nice interpretation as the matrices corresponding to the identity operations of an affine transformation.
%$\{ \mat{X}_{\tfVec}, \mat{Y}_{\tfVec} \}$, $\{ \mat{X}_{\tfAdd}, \mat{Y}_{\tfAdd} \}$ r of size $D \times $

% We now make an observation of our framework in terms of the required \textit{hyperparameter complexity} to perform learning. We consider the highly performing deterministic warping approach of Deep Kernel Learning \cite{wilson2016deep} (DKL). In DKL a large neural network is used to encode the input data before the outputs are passed into a standard GP. The neural network is completely deterministic. As seen in Table \autoref{table:experiments_main}, DKL offers similar performance to our method and is among the overall best performing. The default network structure is (input, output) layer dimensionalities of $[(D, 1000), (1000,500), (500, 50), (50, 2)]$. In addition to this, there are GP hyperparameters. If we probe the parameter count for the model used for the \textit{airfoil} dataset, we find there are a total of $531656$ free parameters for the optimizer to tune. The majority of these come from the neural network. For comparable results of our method SSWIM, we use 512 features for both the warped mapping and top-level predictive function. The total number of free parameters including additional kernel and SSGP hyperparameters accounts to $3510$ which is more than 2 orders of magnitude fewer parameters for similar performance. 
% \subsection{Simulated functions}
% \subsubsection{Linearly trending step function}

\subsubsection{Increased warping depth}\label{sec:analysis3layers}
\begin{figure*}[ht]
    % \vspace{-10pt}
  \centering
  \includegraphics[width=\linewidth]{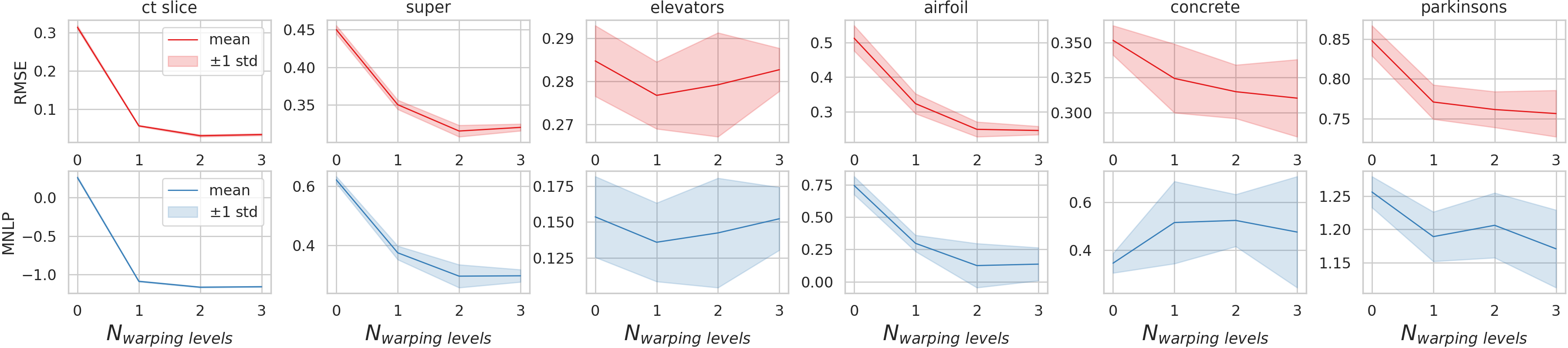}
  \caption{Performance in RMSE and MNLP as the number of warping levels increases. \label{fig:perf_v_layers}}
  \vspace{-5pt}
\end{figure*}
In this experiment we evaluate the predictive performance of SSWIM as we increase the number of levels of consecutive input warping from 0 to 3. A depth of 0 simply corresponds to the \textit{stationary} SSGPR specification. \autoref{fig:perf_v_pseudo} summarises the results. For all the datasets we can see that adding just a single level of input warping increases predictive performance. Adding additional levels of warping seems to consistently improve performance however it adds additional variance to all results which could be explained by the additional complexity required for optimization.

\textbf{Computational complexity}. 
%Throughout this paper we have supported the notion of incorporating uncertainty not just in our top-level predictive model, but also in the input warping model.
The top-level function $\upperGP$ and two warping warping $\tfMat$ and $\tfAdd$ all inherit the same computational complexity of SSGP and SSGP with predictions under uncertainty \cite{pan2017prediction} which is $\mathcal{O}(nm^2 + m^3)$ for $n$ samples and $m$ features. For multiple warping levels this cost is simply multiplied by the number levels $J$ therefore the overall complexity remains $\mathcal{O}(nm^2 + m^3)$. In practice $m$ is very small with $m < 1000$. For SSWIM, the dimensionality of a single pseudo-training point is $D$ which is the same dimensionality of the raw input $\rawInput$. Therefore $\tfMat$ and $\tfAdd$ consist of $D \times N_{\tfMat}$ and $D \times N_{\tfAdd}$ pseudo-training points respectively. The functions $\upperGP$, $\tfMat$ and $\tfAdd$ contain model and kernel hyperparameters of size $|\vec\theta_{\upperGP}|$, $|\vec\theta_{\tfMat}|$ and $|\vec\theta_{\tfAdd}|$ respectively which should each not exceed much more than $D$ for conventional stationary kernels. 

\subsection{Real datasets}\label{sec:experiments_main}
We compare our model on various real-world datasets including multiple regression tasks \cite{Dua:2019, ltorgo:2019, cole2000modelling}. All datasets are standardised using the train set. We use $\frac{2}{3}$ of the samples for training and the remaining $\frac{1}{3}$ for testing. We compare multiple related algorithms alongside our proposed method SSWIM using both one level of warping (SSWIM$_1$) and two levels of warping (SSWIM$_2$), Deep Kernel Learning \cite{wilson2016deep} (DKL), SSGP with stationary random fourier features kernel (RFFS), SSGP with nonstationary kernel features (RFFNS) with freely variable mean and width for the Matern $\frac{3}{2}$ spectral frequencies \cite{remes2017non, ton2018spatial}, Sparse Gaussian Process Regression (SGPR) \cite{titsias2009variational}, Sparse Variational Gaussian Process (SVGP) \cite{hensman2014scalable}, and Doubly Stochastic Deep GP with 2 layers (DSDGP) \cite{salimbeni2017doubly}. All experiments were performed on a Linux machine with a single Titan V GPU. We ran all methods for 150 iterations with stochastic gradient descent and the library GPyTorch was used for DKL, DSDGP, SGPR, and SVGP. We have provided implementations for RFFS, RFFNS, and SSWIM. Full experimental details are provided in the supplementary material with additional commentary. PyTorch Code is provided to reproduce the experimental results. 

In the main experimental results given in \autoref{table:experiments_main} we can observe a consistent high performance across all datasets for SSWIM in all tasks for the RMSE metric. For \textit{concrete, parkinsons} and \textit{bikeshare} SSWIM is outperformed in MNLP by DSDGP, SVGP and RFFS suggesting they were more capable of representing the predictive distribution rather than the mean. For the remaining datasets SSWIM has performed extremely well, most notably on the high dimensional problem \textit{ct slice}. SSWIM$_2$ with two levels of warping comprehensively outperforms other methods as well as SSWIM$_1$ which also performs competitively. These results further corroborate the analysis results given in \autoref{fig:perf_v_pseudo} and \autoref{fig:perf_v_layers}.

 \subsection{Related work}
Foundational work \cite{higdon1999non, paciorek2004nonstationary} on kernel based nonstationarity necessitated manipulation of the kernel function with expensive inference procedures. 
% \textit{Nonstationarity through kernel functions}. 
Recent spectral representation of kernel functions have emerged with Bochner's theorem \cite{bochner1932vorlesungen}. In this paradigm, one constructs kernels in the Fourier domain via \emph{random Fourier features} (RFFs) \cite{Rahimi2008, Rahimi2009} and extensions for nonstationarity via the generalised Fourier inverse transform \cite{samo2015generalized, remes2017non, ton2018spatial, sun2018differentiable}. While general, these methods suffer from various drawbacks such as expensive computations and overfitting due to over-parameterised models \cite{ton2018spatial}.

More expressive modelling frameworks  \cite{calandra2016manifold, wilson2012gaussian, sampson1992nonparametric, anderes2008estimating} have played a major role in expanding the efficacy of kernel based learning. Perhaps the most well known in the recent literature is Deep Kernel Learning \citet{wilson2016deep} and the \emph{deep Gaussian process} \cite{damianou2013deep} and heretofore its various extensions \cite{salimbeni2017doubly, cutajar2017random, bui2016deep}. While functionally elegant, methods like DKL and DGP often rely on increasing the complexity of the composition to produce expressiveness and are often unsuitable or unwieldy in practice occasionally resulting in performance worse than stationary inducing point GPs \cite{salimbeni2017doubly}. We remark a notable difference between DGP and SSWIM is one should interpret our pseudo-training points as hyperparameters of the kernel as opposed to parameters of a variational approximation.   
% This has motivated new research \cite{salimbeni2017doubly,hegde2019deep} on how to overcome various issues such as rank pathologies inherent in deep models including DGPs \cite{duvenaud2014avoiding}. 
% Alternative functional forms have also been proposed including the treed GP \cite{gramacy2005bayesian} which combines piecewise combinations of individual GPs with differing covariance matrices.

Simple bijective input warpings were considered in \cite{snoek2014input} for transforming nonstationary functions into more well behaved functions. In \cite{heinonen2016non} the authors augment the standard GP model by learning nonstationary data dependent functions for the \textit{hyperparameters} of a nonstationary squared exponential kernel \cite{gibbs1997bayesian} however is limited to low dimensions. More recently, the work of \cite{hegde2019deep} has explored a dynamical systems view of input warpings by processing the inputs through a time dependent differential fields. Less related models presented in \citet{Wang2012, Dutordoir2018cde, snelson2004warped} involve \textit{output warping} non-Gaussian likelihoods and heteroscedastic noise. For the curious reader we examine contrasting properties of output  and input warping in the supplementary material.
 
\section{Conclusion}
We have proposed a crucial advance to the sparse spectrum Gaussian process framework to account for nonstationarity through a novel input warping formulation. We introduced a novel form of input warping analytically incorporating complete Gaussian measures in the functional warping with the concept of \textit{pseudo-training} data and latent \textit{self-supervision}. We have further extended this core contribution with the necessary results to extend the warping to multiple levels resulting in higher levels of model expressiveness.

Experimentally, the methodology we propose has demonstrated excellent results in the total number of hyperparameters for various low and high dimensional real world datasets when compared to deterministic and neural network based approaches but also performing exceptionally well in contrast to deep Gaussian processes. Our model suggests an interesting and effective inductive bias this is nicely interpreted as a learned conditional affine transformation. This perspectives invites a fresh take on how we can discover more effective representations of nonstationary data.

% Overall we believe that accounting for uncertainty of the input warping function at the prediction level is a necessary step forward for learning more expressive manifold models. Finally, we have demonstrated on synthetic and real datasets the efficacy of the new warping representation through improved learning.

\newpage
\section*{Broader Impact}

The problem we address in this paper of efficient modelling of nonstationary stochastic processes is fundamental in geostatistics, time-series analysis, and the study of dynamical systems. To this end, our technique is directly applicable to spatial-temporal problems such as air pollution monitoring, the spread of diseases, and the study of natural resources such as underground water. In all of these problems, the strength of the spatial relationships between inputs varies with respect to the location. For example, during the current pandemic, nearby cities might exhibit different levels of infection rates within their boundaries but still being related due to infected people travelling between them~\citet{senanayake2016predicting}. Our approach is directly applicable to these cases and can be incorporated within epidemiological models that typically aggregate populations in large regions for a more refined prediction and study of intervention policies such as social distancing.

% =============================
% Authors are required to include a statement of the broader impact of their work, including its ethical aspects and future societal consequences. 
% Authors should discuss both positive and negative outcomes, if any. For instance, authors should discuss a) 
% who may benefit from this research, b) who may be put at disadvantage from this research, c) what are the consequences of failure of the system, and d) whether the task/method leverages
% biases in the data. If authors believe this is not applicable to them, authors can simply state this.

% Use unnumbered first level headings for this section, which should go at the end of the paper. {\bf Note that this section does not count towards the eight pages of content that are allowed.}

% \begin{ack}
% Use unnumbered first level headings for the acknowledgments. All acknowledgments
% go at the end of the paper before the list of references. Moreover, you are required to declare 
% funding (financial activities supporting the submitted work) and competing interests (related financial activities outside the submitted work). 
% More information about this disclosure can be found at: \url{https://neurips.cc/Conferences/2020/PaperInformation/FundingDisclosure}.

% Do {\bf not} include this section in the anonymized submission, only in the final paper. You can use the \texttt{ack} environment provided in the style file to autmoatically hide this section in the anonymized submission.
% \end{ack}

% \newpage
% \section*{References}
\medskip
\small
\bibliographystyle{unsrtnat}
\bibliography{main}

\end{document}

\maketitle

\appendix

\section{Mathematical derivations}

\begin{lemma}
\label{thr:warp-uncetainty}
Let $\anyvector \sim \normal(\vecMean{\anyvector}, \mat\Sigma_\anyvector)$ denote a Gaussian random vector and $\rawInput\in\R^\dimension$ an arbitrary point. Then $\othervector := \anyvector\odot\rawInput$ is Gaussian, $\othervector\sim\normal(\vecMean{\othervector}, \mat\Sigma_\othervector)$, with mean and covariance matrix given by:
\begin{align}
    \vecMean{\othervector} &= \vecMean{\anyvector}\odot\rawInput\\
    \mat\Sigma_\othervector &= \rawInput\vec{1}^\transpose\odot\mat\Sigma_\anyvector\odot\vec{1}\rawInput^\transpose
\end{align}
\end{lemma}
\begin{proof}
The element-wise vector product, which is the Hadamard product for single-column matrices, is symmetric and linear, i.e. $\vec a \odot \vec b = \vec b \odot \vec a$ and $\vec a \odot (\vec b + \vec c) = \vec a \odot \vec b + \vec a \odot \vec c$, for any $\vec a, \vec b, \vec c \in \R^\dimension$. By the linearity of the expectation, we then have $\expectation[\othervector] = \expectation[\anyvector\odot\rawInput] = \expectation[\anyvector]\odot\rawInput$. From the properties of the Hadamard product, one can also show that $\vec a (\vec b \odot \vec c)^\transpose = \vec a \vec b^\transpose \odot \vec{1}\vec c^\transpose$, where $\vec 1$ is a vector of 1's. Therefore, the covariance matrix $\mat\Sigma_\othervector := \variance[\othervector] =  \expectation[(\othervector-\vecMean{\othervector})(\othervector-\vecMean{\othervector})^\transpose]$, is given by:
\begin{equation}
    \begin{split}
        \variance[\othervector] &= \expectation[((\anyvector-\vecMean{\anyvector})\odot\rawInput)((\anyvector-\vecMean{\anyvector})\odot\rawInput)^\transpose]\\
    &=\rawInput\vec 1^\transpose\odot\expectation[(\anyvector-\vecMean{\anyvector})(\anyvector-\vecMean{\anyvector})^\transpose]\odot \vec 1 \rawInput^\transpose\\
    &= \rawInput \vec 1^\transpose \odot \variance[\anyvector]\odot \vec 1 \rawInput^\transpose~,
    \end{split}
\end{equation}
which concludes the proof.
\end{proof}

\section{Relationship to output warped GPs} %TODO: mention heteroscedasticity
\label{sec:experiments_outputwarp}
\begin{table}[ht]
\small
\centering
\caption{MSE and MNLP metrics for comparison with Warped and Bayesian Warped GPs \cite{lazaro2012bayesian}. MSE results for \textit{ailerons} are $\times 10^{-8}$. }
% \begin{tabular}{c|l|llllllllll}
\begingroup
\setlength{\tabcolsep}{4.5pt} % Default value: 6pt
\renewcommand{\arraystretch}{1.0} % Default value: 1
% \begin{tabular}{l|l|llllllllll}
\begin{tabular}{l|l|ccc}
\toprule
 \multicolumn{1}{c}{}&\multicolumn{1}{c}{Method}     &\multicolumn{1}{c}{abalone}         &\multicolumn{1}{c}{creep}         &\multicolumn{1}{c}{ailerons}                      \\\hline\parbox[t]{1mm}{\multirow{6}{*}{\rotatebox[origin=c]{90}{ 
MSE }}}  
   &GP        & 4.55 $\pm$ 0.14 & 584.9 $\pm$ 71.2   & 2.95 $\pm$ 0.16                        \\
 &BWGP      & 4.55 $\pm$ 0.11 & 491.8 $\pm$ 36.2   & 2.91 $\pm$ 0.14                          \\
 &MLWGP3    & 4.54 $\pm$ 0.10 & 502.3 $\pm$ 43.3   & \textbf{2.80 $\pm$ 0.11}             \\
 &MLWGP20   & 4.59 $\pm$ 0.32 & 506.3 $\pm$ 46.1   & 3.42 $\pm$ 2.87                           \\
 &SSWIM$_1$ & 4.64 $\pm$ 0.13 & 483.69 $\pm$ 64.12 & 2.96 $\pm$ 0.08                          \\
 &SSWIM$_2$ & \textbf{4.50 $\pm$ 0.11} & \textbf{279.86 $\pm$ 31.88} & 2.83 $\pm$ 0.06          \\
 \midrule \parbox[t]{1mm}{\multirow{6}{*}{\rotatebox[origin=c]{90}{MNLP}}} 
   &GP      & 2.17 $\pm$ 0.01 & 4.46 $\pm$ 0.03 & -7.30 $\pm$ 0.01              \\
 &BWGP           & 1.99 $\pm$ 0.01 & 4.31 $\pm$ 0.04 & -7.38 $\pm$ 0.02              \\
 &MLWGP3          & \textbf{1.97 $\pm$ 0.02} & \textbf{4.21 $\pm$ 0.03 }& -7.44 $\pm$ 0.01 \\
 &MLWGP20      & 1.99 $\pm$ 0.05 & \textbf{4.21 $\pm$ 0.08 }& \textbf{-7.45 $\pm$ 0.08}              \\
 &SSWIM$_1$     & 2.18 $\pm$ 0.01 & 4.45 $\pm$ 0.03 & -7.24 $\pm$ 0.01              \\
 &SSWIM$_2$      & 2.17 $\pm$ 0.02 & 4.27 $\pm$ 0.03 & -7.00 $\pm$ 0.02              \\
\bottomrule
\end{tabular}
\label{table:experiments_outputwarp}
\endgroup
\end{table}

A body of work has previously been developed under the title of \textit{warped} Gaussian processes \cite{snelson2004warped}. As noted, this contrasts from our modelling problem because while we proceed to expand the GP's capabilities to warp the \textit{inputs}, the WGP and extensions warps explicitly the \textit{output} distribution of the Gaussian process. We now juxtapose the efficacy of our input warping formulation with WGP by applying SSWIM to the three challenging datasets \textit{abalone}, \textit{creep}, and \textit{ailerons} experimented upon in \cite{snelson2004warped, lazaro2012bayesian}. 

Our results in Table \autoref{table:experiments_outputwarp} suggest that with SSWIM we are able to improve upon both WGP and BWGP in MSE: marginal improvements for \textit{abalone} and a signficant improvement for \textit{creep} with comparable performance in \textit{ailerons}. Contrasting this, the output warping methods outperform unanimously on the MNLP metric. This is expected because output warping may allow one to capture non-gaussian conditional distributions which an input warping formulation cannot with a standard Gaussian process. The discussion we wish to raise here is that both aspects of manipulating the inputs and outputs of a GP can result in major improvements respectively across different metrics.

\section{Additional Experiments}
\subsection{Increasing number of pseudo-training points}
For the "increasing number of pseudo-training points" experiment we used 1 layer of warping with 256 features for both the warping and top-level predictive functions. 

\subsection{Increasing warping depth}
We used 256 features and 1280 pseudo-training points for all of the experiments. 

\subsection{Complete real-dataset experiments table}
\autoref{table:experiments_main} contains additional real-world experiments to extend the majore experimental results from the main paper.

\begin{table*}[ht]
% \tiny 
\scriptsize
\centering
\caption{RMSE and MNLP metrics for various real world datasets.}
% \begin{tabular}{c|l|llllllllll}
\begingroup
\setlength{\tabcolsep}{2.75pt} % Default value: 6pt
\renewcommand{\arraystretch}{1.0} % Default value: 1
% \begin{tabular}{l|l|llllllllll}
\begin{adjustbox}{center}
\begin{tabular}{l|l|cccccccccc}
\toprule
 \multicolumn{1}{c}{$\quad$}&\multicolumn{1}{c}{\textit{\tiny{$(D, N)$}}}     &\multicolumn{1}{c}{\tiny{$(18, 8751)$}}         &\multicolumn{1}{c}{\tiny{$(5, 1503)$}}         &\multicolumn{1}{c}{\tiny{$(8, 1030)$}}         &\multicolumn{1}{c}{\tiny{$(16, 5875)$}}       &\multicolumn{1}{c}{\tiny{$(15, 17379)$}}         &\multicolumn{1}{c}{\tiny{$(379, 53500)$}}          &\multicolumn{1}{c}{\tiny{$(81,21263)$}}       &\multicolumn{1}{c}{\tiny{$(9,45730)$}}          &\multicolumn{1}{c}{\tiny{$(77,583250)$}}             &\multicolumn{1}{c}{\tiny{$(90,515345)$}}              \\
 \multicolumn{1}{c}{$\quad$}&\multicolumn{1}{c}{Method}     &\multicolumn{1}{c}{elevators}         &\multicolumn{1}{c}{airfoil}         &\multicolumn{1}{c}{concrete}         &\multicolumn{1}{c}{parkinsons}       &\multicolumn{1}{c}{bikeshare}         &\multicolumn{1}{c}{ct slice}          &\multicolumn{1}{c}{supercond}       &\multicolumn{1}{c}{protein}          &\multicolumn{1}{c}{buzz}             &\multicolumn{1}{c}{song}              \\
 \hline\parbox[t]{1mm}{\multirow{8}{*}{\rotatebox[origin=c]{90}{RMSE ($\times 10^{-1}$)}}}  
 &SSWIM$_1$ & 2.83 $\pm$ 0.08   & 2.38 $\pm$ 0.26 & 3.05 $\pm$ 0.26  & 7.63 $\pm$ 0.20  & 0.13 $\pm$ 0.04   & 0.46 $\pm$ 0.02   & 3.44 $\pm$ 0.14 & 5.91 $\pm$ 0.07  & 2.98 $\pm$ 0.04  & 8.12 $\pm$ 0.05   \\
 &SSWIM$_2$ & \textbf{2.74 $\pm$ 0.08}   & \textbf{2.35 $\pm$ 0.22} &\textbf{ 3.01 $\pm$ 0.31 } & \textbf{7.55 $\pm$ 0.15}  & 0.11 $\pm$ 0.03   & \textbf{0.23 $\pm$ 0.01}   & \textbf{3.02 $\pm$ 0.04} & \textbf{5.80 $\pm$ 0.08}  & \textbf{2.40 $\pm$ 0.01}  & \textbf{7.97 $\pm$ 0.03}   \\
 &DSDGP     & \textbf{2.74 $\pm$ 0.06}   & 4.30 $\pm$ 0.19 & 5.88 $\pm$ 1.24  & 7.94 $\pm$ 0.20  & 0.33 $\pm$ 0.55   & 4.81 $\pm$ 1.18   & 5.10 $\pm$ 0.84 & 5.96 $\pm$ 0.06  & 3.65 $\pm$ 0.75  & 8.46 $\pm$ 0.03   \\
 &DKL       & 3.06 $\pm$ 0.29   & 3.19 $\pm$ 0.37 & 3.18 $\pm$ 0.38  & 8.84 $\pm$ 0.74  & 0.24 $\pm$ 0.03   & 0.52 $\pm$ 0.08   & 3.46 $\pm$ 0.18 & 7.15 $\pm$ 1.10  & 4.11 $\pm$ 3.33  & 16.66 $\pm$ 8.14  \\
 &RFFNS     & 2.83 $\pm$ 0.07   & 3.31 $\pm$ 0.37 & 3.46 $\pm$ 0.24  & 8.15 $\pm$ 0.15  & 0.05 $\pm$ 0.01  & 4.39 $\pm$ 0.27   & 3.85 $\pm$ 0.05 & 6.87 $\pm$ 0.06  & 5.70 $\pm$ 0.84  & 8.35 $\pm$ 0.03   \\
 &SVGP      & 2.88 $\pm$ 0.10   & 2.70 $\pm$ 0.15 & 3.32 $\pm$ 0.26  & 8.14 $\pm$ 0.12  & 0.06 $\pm$ 0.03   & 1.16 $\pm$ 0.02   & 4.06 $\pm$ 0.05 & 7.32 $\pm$ 0.08  & 9.98 $\pm$ 0.02  & 12.19 $\pm$ 0.18  \\
 &SGPR      & 4.96 $\pm$ 2.02   & 4.24 $\pm$ 0.40 & 5.55 $\pm$ 0.58  & 7.86 $\pm$ 0.22  & 0.67 $\pm$ 0.18   & 1.79 $\pm$ 0.04   & 4.27 $\pm$ 0.06 & 6.45 $\pm$ 0.07  & 2.89 $\pm$ 0.02  & 8.40 $\pm$ 0.04   \\
 &RFFS      & 2.87 $\pm$ 0.10   & 3.28 $\pm$ 0.24 & 3.33 $\pm$ 0.30  & 8.24 $\pm$ 0.17  & \textbf{0.03 $\pm$ 0.00}   & 2.34 $\pm$ 0.05   & 3.89 $\pm$ 0.06 & 6.91 $\pm$ 0.07  & 3.78 $\pm$ 0.14  & 8.36 $\pm$ 0.04   \\
 \midrule \parbox[t]{1mm}{\multirow{8}{*}{\rotatebox[origin=c]{90}{MNLP ($\times 10^{-1}$)}}} 
 &SSWIM$_1$ & 1.81 $\pm$ 0.55   & 1.25 $\pm$ 2.50 & 10.22 $\pm$ 4.15 & 11.95 $\pm$ 0.47 & -11.89 $\pm$ 0.15 & -11.24 $\pm$ 0.05 & 3.55 $\pm$ 0.32 & 8.95 $\pm$ 0.12  & 2.03 $\pm$ 0.13  & 12.08 $\pm$ 0.05  \\
 &SSWIM$_2$ & 1.65 $\pm$ 0.63   & \textbf{1.05 $\pm$ 1.74 }& 5.19 $\pm$ 2.59  & 12.50 $\pm$ 0.44 & -11.78 $\pm$ 0.07 & \textbf{-11.79 $\pm$ 0.02} & \textbf{2.82 $\pm$ 0.29} & \textbf{8.82 $\pm$ 0.15}  & \textbf{-0.09 $\pm$ 0.04} & \textbf{11.93 $\pm$ 0.04}  \\
 &DSDGP     & \textbf{1.16 $\pm$ 0.21}   & 9.19 $\pm$ 0.20 & 11.02 $\pm$ 1.06 & \textbf{11.91 $\pm$ 0.24} & -23.28 $\pm$ 8.29 & 6.62 $\pm$ 2.61   & 7.36 $\pm$ 1.62 & 9.04 $\pm$ 0.10  & 3.80 $\pm$ 2.02  & 12.52 $\pm$ 0.04  \\
 &DKL       & 7.57 $\pm$ 0.15   & 7.70 $\pm$ 0.17 & 7.69 $\pm$ 0.20  & 13.17 $\pm$ 1.12 & 6.82 $\pm$ 0.01   & 6.83 $\pm$ 0.01   & 7.76 $\pm$ 0.10 & 11.02 $\pm$ 1.53 & 9.01 $\pm$ 4.65  & 42.64 $\pm$ 44.77 \\
 &RFFNS     & 1.44 $\pm$ 0.24   & 2.74 $\pm$ 1.20 & 3.31 $\pm$ 0.45  & 12.18 $\pm$ 0.18 & -11.97 $\pm$ 0.00 & 5.95 $\pm$ 0.66   & 4.66 $\pm$ 0.12 & 10.39 $\pm$ 0.08 & 8.78 $\pm$ 1.87  & 12.39 $\pm$ 0.04  \\
 &SVGP      & 1.78 $\pm$ 0.30   & 1.19 $\pm$ 0.33 & \textbf{2.83 $\pm$ 0.56}  & 12.21 $\pm$ 0.14 & \textbf{-27.70 $\pm$ 1.24} & -5.98 $\pm$ 0.13  & 5.32 $\pm$ 0.12 & 11.10 $\pm$ 0.09 & 63.31 $\pm$ 3.44 & 18.02 $\pm$ 0.09  \\
 &SGPR      & 11.59 $\pm$ 22.38 & 6.45 $\pm$ 0.47 & 8.48 $\pm$ 2.10  & 12.39 $\pm$ 0.30 & -13.67 $\pm$ 0.98 & -3.14 $\pm$ 0.26  & 5.58 $\pm$ 0.10 & 10.05 $\pm$ 0.13 & 1.11 $\pm$ 0.12  & 11.97 $\pm$ 0.07  \\
 &RFFS      & 1.59 $\pm$ 0.26   & 2.72 $\pm$ 0.69 & 3.05 $\pm$ 0.96  & 12.29 $\pm$ 0.20 & -11.98 $\pm$ 0.00 & -0.33 $\pm$ 0.22  & 4.79 $\pm$ 0.13 & 10.45 $\pm$ 0.09 & 4.41 $\pm$ 0.37  & 12.40 $\pm$ 0.05  \\
\bottomrule
\end{tabular}
\end{adjustbox}
\endgroup
\label{table:experiments_main}
\end{table*}

\subsection{Extended discussion}
It is imperative to note here our aim is not to demand any algorithmic dominance when comparing methods. Firstly, this is a fruitless pursuit due to the diversity and dependence of all advanced GP methods on hyperparameter optimization, and secondly it is not a constructive approach to the communal development of new methodologies to claim some benchmark task superiority. Rather, encouraged by the results of this paper, we invite the discussion to move beyond stationary kernels and inquire upon and interpret the effectiveness of new Gaussian process methodologies through the \textit{more general perspective} of input space warping for unearthing hidden nonstationary patterns within data.

\subsection{Overfitting analysis}
We ran an overfitting analysis of $SSWIM_1$ to observe the effect of over-optimising with respect to the marginal likelihood. We ran with 256 features, 1280 pseudo-training points, for 150 steps, with 10 repeats, and evaluated the test RMSE and test MNLP on the test set for every single epoch of optimisation. This test bench is provided in the supplementary code. The results are averaged with mean and standard deviations of the training curves displayed in \autoref{fig:analysis_overfitting}. We can see that we are quite resistant to overfitting except for RMSE in the \textit{elevators} dataset and the MNLP in the \textit{concrete} and \textit{parkinsons} datasets. The causes of this could be explained by the underlying flexibility of the proposed method which allows the model to become overconfident in what it has learned with respect to the data it has observed. In fact, we observed similar overfitting behaviour for similar optimisation periods with DKL and DSDGP. 

This analysis leads to some interesting observations and recommendations for future algorithm development in more expressive GP methodologies: 1. the marginal likelihood is no panacea although it is easy to think it is, and 2. other loss functions and training schemes, such as leave-one-out cross validation. Actually, these results corroborate long known discussions from \cite{rasmussen2004gaussian} about the risk of overfitting from trusting the marginal likelihood with standard optimisation procedures, however their importance seems to have been largely ignored in evaluation of recent methodology innovations in the GP literature. We believe that a more open discussion should be on the table for analysing the interplay between model expressiveness and the effect this has on overfitting; this is especially pertinent to the GP literature which has placed a large emphasis on the importance of the marginal likelihood has a valid hyperparameter optimisation loss.

\begin{figure*}[ht]
  \centering
  \includegraphics[width=1.0\textwidth]{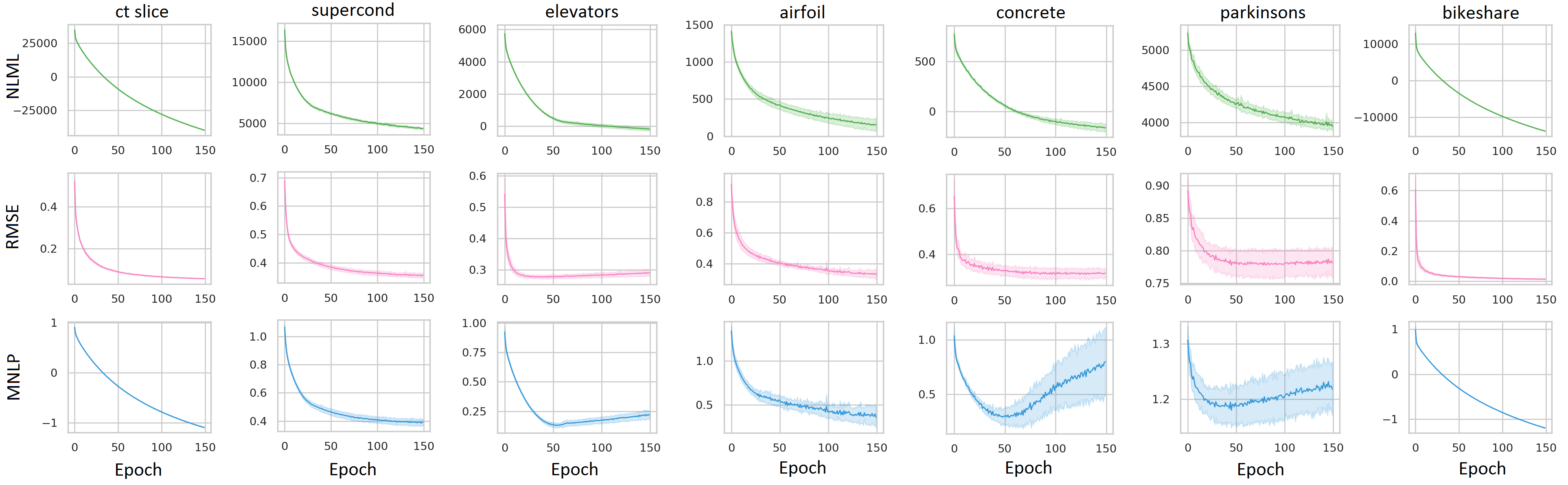}
  \caption{Empirical analysis of overfitting behaviour in SSWIM\label{fig:analysis_overfitting}}
\end{figure*}

\subsection{Pseudo-training points in $2D$}
\begin{figure*}[ht]
  \centering
  \includegraphics[width=1.0\textwidth]{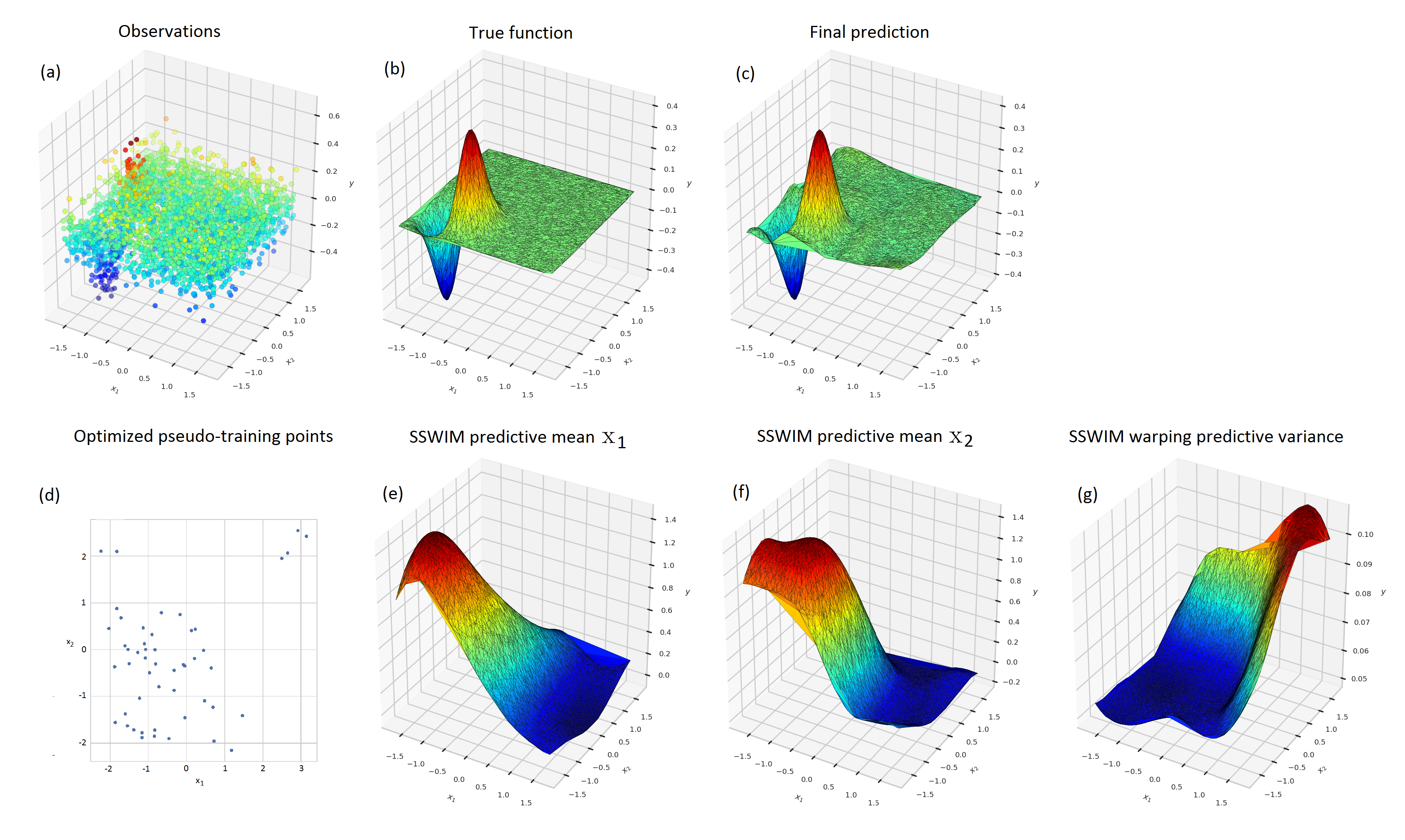}
  \caption{Visualisation of learning spatial positioning of pseudo training points in 2D. We demonstrate spatial nonstationarity with the exponential 2D function from \cite{gramacy2005bayesian}. (a) Noisy training data, (b) True function surface, (c) SSWIM prediction conditioned on training data, (d) Learned pseudo-training point positions, (e) Learned warping predictive mean for $x_1$, (f) Learned warping predictive mean for $x_2$, (g) Learned warping predictive variance.\label{fig:gramacy_vis}}
\end{figure*}

\autoref{fig:gramacy_vis} contains an interpretation of input warping and the pseudo-training points in higher dimensions. We use the exponential 2D function from \cite{gramacy2005bayesian} as a case where it is intuitive to observe how SSWIM reacts to topological differences in the underlying function. The function consists of a mostly flat surface with abrupt spiking occurring at a corner of the domain as seen in \autoref{fig:gramacy_vis} (b). If we were to assume a homogeneous domain, and model our data with a standard SSGP with stationary RBF kernel, the kernel's hyperparameters would be optimised to provide a homogeneous representation resulting in conflict between the spiked area in the corner and flat areas elsewhere. By directly manipulating the input domain with our warping, we are able to transform the input domain with a continuous warping that consequently allows accurate representation as seen in \autoref{fig:gramacy_vis} (c).

The intuition and utility of our proposed \textit{pseudo-training} as virtual training points is visualised from a birds-eye view in \autoref{fig:gramacy_vis} (d). Before optimisation, we initialise these points uniformly across the domain of the training space. It is clear that the learned positions of the pseudo-training data have transformed their spatial locations away from uniform. At the bottom left they appear to have clustered near the discontinuity of the test function while in the remaining corners of the space the points have spread away. The predictive mean of the learned warping function is thus visualised across the domain, for both $x_1$ and $x_2$, in \autoref{fig:gramacy_vis} (e) and (f). Furthermore, \autoref{fig:gramacy_vis} (g) shows the predictive variance of the warping function and we can see how a lower amount of spatial uncertainty arises both from the noisy pseudo-targets as well as the pseudo-inputs from (d).

\section{Datasets and Experimental Conditions}
For completeness, we specify for each dataset used, the data dimensionality and sample size, the raw data source, the modelling objective (i.e. the target) as defined for the original problem, and any target variable pre-processing excluding standardisation. The dimensionality $D$ is reported for the inputs $X$ (i.e. excluding the target variable $y$). All problems are single output regression tasks. Number of samples $N$ is reported for the entire dataset before train/test splitting is applied. Note that we do not alter the raw data files and pre-processing is applied through code exactly as in the provided supplementary code. \textit{dataloader.py}.

\subsubsection*{elevators}
$D: 18$\\
$N: 8751$\\
\textit{Source:} \url{https://web.archive.org/web/*/http://www.liacc.up.pt/~ltorgo/Regression/*}\\
\textit{Preprocessing: } None \\
\textit{Target: } goal

\subsubsection*{airfoil}
$D: 5$\\
$N: 1503$\\
\textit{Source:} \url{https://archive.ics.uci.edu/ml/datasets/Airfoil+Self-Noise}\\
\textit{Preprocessing: } None \\
\textit{Target: } Sound pressure in decibels

\subsubsection*{concrete}
$D: 8$\\
$N: 1030$\\
\textit{Source:} \url{https://archive.ics.uci.edu/ml/datasets/concrete+compressive+strength}\\
\textit{Preprocessing: } None \\
\textit{Target: } Compressive Strength

\subsubsection*{parkinsons}
$D: 16$\\
$N: 5875$\\
\textit{Source:} \url{https://archive.ics.uci.edu/ml/datasets/Parkinsons+Telemonitoring}\\
\textit{Preprocessing: } Drop the first 5 columns as they are not used in the original problem \\
\textit{Target: } Total udpr

\subsubsection*{bikeshare}
$D: 15$\\
$N: 17379$\\
\textit{Source:} \url{https://archive.ics.uci.edu/ml/datasets/Bike+Sharing+Dataset}\\
\textit{Preprocessing: } None \\
\textit{Target: } Number of bike shares per hour 

\subsubsection*{ct slice}
$D: 379$\\
$N: 53500$\\
\textit{Source:} \url{https://archive.ics.uci.edu/ml/datasets/Relative+location+of+CT+slices+on+axial+axis}\\
\textit{Preprocessing: } Drop Patient ID. Drop columns which have constant value throughout entire dataset. \\
\textit{Target: } Reference (relative location)

\subsubsection*{supercond}
$D: 81$\\
$N: 21263$\\
\textit{Source:} \url{https://archive.ics.uci.edu/ml/datasets/Superconductivty+Data}\\
\textit{Preprocessing: } None \\
\textit{Target: } Critical Temperature

\subsubsection*{protein}
$D: 9$\\
$N: 45730$\\
\textit{Source:} \url{https://archive.ics.uci.edu/ml/datasets/Physicochemical+Properties+of+Protein+Tertiary+Structure}\\
\textit{Preprocessing: } $\log(1+y)$ transform for target $y$ \\
\textit{Target: } RMSD

\subsubsection*{buzz}
$D: 77$\\
$N: 583250$\\
\textit{Source:} \url{https://archive.ics.uci.edu/ml/datasets/Buzz+in+social+media+#}\\
\textit{Preprocessing: } $\log(1+y)$ transform of target $y$ \\
\textit{Target: } Mean Number of Active Discussion (NAD)

\subsubsection*{song}
$D: 90$\\
$N: 515345$\\
\textit{Source:} \url{https://archive.ics.uci.edu/ml/datasets/Buzz+in+social+media+#}\\
\textit{Preprocessing: } None \\
\textit{Target: } Year of song release

\subsubsection*{abalone}
$D: 9$\\
$N: 4177$\\
\textit{Source:} \url{https://web.archive.org/web/*/http://www.liacc.up.pt/~ltorgo/Regression/*}\\
\textit{Preprocessing: } None \\
\textit{Target: } Number of Rings

\subsubsection*{creep}
$D: 30$\\
$N: 2066$\\
\textit{Source:} \url{http://www.phase-trans.msm.cam.ac.uk/map/data/materials/creeprupt-b.html#down}\\
\textit{Preprocessing: } None \\
\textit{Target: } Rupture stress

\subsubsection*{ailerons}
$D: 40$\\
$N: 7154$\\
\textit{Source:} \url{https://web.archive.org/web/*/http://www.liacc.up.pt/~ltorgo/Regression/*}\\
\textit{Preprocessing: } None \\
\textit{Target: } goal

\newpage
% \section*{References}
\medskip
\small
\bibliographystyle{ieeetr}
\bibliography{bibli}